\documentclass[lettersize,journal]{IEEEtran}

\usepackage{array}
\usepackage[caption=false,font=normalsize,labelfont=sf,textfont=sf]{subfig}
\usepackage{textcomp}
\usepackage{stfloats}
\usepackage{verbatim}
\usepackage{cite}

\usepackage{amssymb}
\usepackage{amsthm}
\usepackage{graphicx}
\usepackage{amsfonts}
\usepackage{caption}
\usepackage[labelformat=simple]{subcaption}
\usepackage{float}
\usepackage{enumitem}
\usepackage{algorithm} 
\usepackage{algorithmic}
\usepackage[algo2e]{algorithm2e}
\usepackage{amsmath,amsthm,mathtools}
\usepackage{stmaryrd}
\usepackage{multirow}
\usepackage[table,xcdraw]{xcolor}
\usepackage{url}
\usepackage{ragged2e}
\usepackage{epstopdf}

\usepackage{tikz}
\usepackage{threeparttable}
\usepackage{booktabs}
\usepackage{pifont}

\usepackage{import}
\usepackage{xifthen}
\usepackage{pdfpages}
\usepackage{transparent}
\usepackage{hyperref}
\usepackage{lipsum}

\newcommand{\etal}{\textit{et al}. }
\newcommand{\ie}{\textit{i}.\textit{e}., }
\newcommand{\eg}{\textit{e}.\textit{g}., }

\DeclareMathOperator{\diag}{\operatorname{diag}}

\DeclareMathOperator{\tr}{\operatorname{tr}}

\DeclareMathOperator{\softmax}{\operatorname{softmax}}
\DeclareMathOperator{\relu}{\operatorname{ReLU}}

\newtheorem{theorem}{Theorem}
\newtheorem{lemma}{Lemma}

\newtheorem{definition}{Definition}

\begin{document}

\title{Higher-Order GNNs Meet Efficiency: Sparse Sobolev Graph Neural Networks}

\author{Jhony H. Giraldo, Aref Einizade, Andjela Todorovic, Jhon A. Castro-Correa, Mohsen Badiey, Thierry Bouwmans, and Fragkiskos D. Malliaros
\thanks{Jhony H. Giraldo, Aref Einizade, and Andjela Todorovic are with the LTCI, Télécom Paris, Institut Polytechnique de Paris, 91120 Palaiseau, France e-mail:  jhony.giraldo@telecom-paris.fr, aref.einizade@telecom-paris.fr, andela.todorovic@ip-paris.fr.}
\thanks{Jhon A. Castro-Correa and Mohsen Badiey are with the Department of Electrical and Computer Engineering, University of Delaware, Newark, DE, USA. E-mail: jcastro@udel.edu, badiey@udel.edu.}
\thanks{Thierry Bouwmans is with the laboratoire MIA, Mathématiques, Image et Applications, La Rochelle Université, 17000 La Rochelle, France e-mail: tbouwman@univ-lr.fr.}
\thanks{Fragkiskos D. Malliaros is with Université Paris-Saclay, CentraleSupélec, Inria, Centre for Visual Computing (CVN), Gif-Sur-Yvette, France. E-mail: fragkiskos.malliaros@centralesupelec.fr.}
\thanks{This work was supported in part by the ANR (French National Research Agency) under the JCJC project GraphIA (ANR-20-CE23-0009-01), by the Office of Naval Research ONR (Grant No. N00014-21-1-2760), and by the center Hi! PARIS.}
}

\markboth{Signal Processing}%
{Shell \MakeLowercase{\textit{et al.}}: A Sample Article Using IEEEtran.cls for IEEE Journals}


\maketitle

\begin{abstract}
    Graph Neural Networks (GNNs) have shown great promise in modeling relationships between nodes in a graph, but capturing higher-order relationships remains a challenge for large-scale networks.
    Previous studies have primarily attempted to utilize the information from higher-order neighbors in the graph, involving the incorporation of powers of the shift operator, such as the graph Laplacian or adjacency matrix.
    This approach comes with a trade-off in terms of increased computational and memory demands.
    Relying on graph spectral theory, we make a fundamental observation: \textit{the regular and the Hadamard power of the Laplacian matrix behave similarly in the spectrum}.
    This observation has significant implications for capturing higher-order information in GNNs for various tasks such as node classification and semi-supervised learning.
    Consequently, we propose a novel graph convolutional operator based on the sparse Sobolev norm of graph signals.
    Our approach, known as Sparse Sobolev GNN (S2-GNN), employs Hadamard products between matrices to maintain the sparsity level in graph representations.
    S2-GNN utilizes a cascade of filters with increasing Hadamard powers to generate a diverse set of functions.
    We theoretically analyze the stability of S2-GNN to show the robustness of the model against possible graph perturbations.
    We also conduct a comprehensive evaluation of S2-GNN across various graph mining, semi-supervised node classification, and computer vision tasks.
    In particular use cases, our algorithm demonstrates competitive performance compared to state-of-the-art GNNs in terms of performance and running time.
\end{abstract}


\section{Introduction}
\label{sec:introduction}

\IEEEPARstart{G}{raph} representation learning and its applications have garnered significant attention in recent years \cite{bruna2014spectral,defferrard2016convolutional,kipf2017semi,velickovic2018graph}.
Notably, Graph Neural Networks (GNNs) have emerged as a powerful extension of Convolutional Neural Networks (CNNs) for modeling non-Euclidean data as graphs.
GNNs have been successfully applied in various domains, including semi-supervised learning \cite{kipf2017semi}, clustering \cite{duval2022higherorder}, point cloud semantic segmentation \cite{li2019deepgcns}, misinformation detection \cite{benamira2019semisupervised}, and molecular modeling \cite{faenet-icml23}, among others.

Most GNNs update their node embeddings by performing specific operations within the neighborhood of each node through message passing \cite{zhang2022deep}.
However, this updating rule has limitations when it comes to capturing higher-order\footnote{We use the term ``higher-order" to refer to higher-order hop information (more than one hop).} relationships between nodes since it only leverages information from the immediate neighbors (1-hop) of each vertex.
To address this limitation, previous methods in GNNs have attempted to capture higher-order connections by incorporating powers of the sparse shift operator of the graph, such as the adjacency or Laplacian matrix \cite{frasca2020sign,defferrard2016convolutional}.
Nevertheless, these methods encounter computational and memory issues, attributed either to the densification of the shift operator \cite{frasca2020sign} or a bottleneck during the implementation of recursive graph diffusion operations \cite{defferrard2016convolutional}.


This work explores a somewhat radical departure from prior approaches in higher-order graph convolutions for GNNs.
We observe that the eigenvalues of the higher-order Laplacian matrix in a weighted graph exhibit similar behavior when subjected to regular and Hadamard power operations (the Hadamard power is also known as element-wise power).
To establish the precise relationship between the spectra of the regular and sparse higher-order Laplacian matrix, we rely on tools from spectral graph theory \cite{chung1997spectral} and the Schur product theorem \cite{horn2012matrix}.
This fundamental observation enables us to design accurate and efficient sparse graph convolutions for GNNs.
More specifically, we propose a novel approach called Sparse Sobolev GNN (S2-GNN), which employs a cascade of sparse higher-order filtering operations.
This allows for the computation of a more diverse set of functions at each layer.
To achieve this goal, we introduce a new sparse Sobolev norm, drawing inspiration from the literature of Graph Signal Processing (GSP) \cite{pesenson2009variational,giraldo2022reconstruction}.

In S2-GNN, the information of the higher-order filtering operations is aggregated using either a linear combination or a Multi-Layer Perceptron (MLP) fusion layer.
We conduct thorough evaluations of S2-GNN in various semi-supervised learning tasks across several domains, including tissue phenotyping in colon cancer histology images \cite{kather2016multi}, text classification of news \cite{lang1995newsweeder}, activity recognition with sensors \cite{anguita2013public}, and recognition of spoken letters \cite{fanty1991spoken}.
Furthermore, we test our algorithm in node classification tasks using benchmark networks commonly employed in the literature \cite{kipf2017semi,hu2020open}.
S2-GNN demonstrates superior or competitive performance compared to a broad range of GNN methods that employ first-order or higher-order graph convolutions \cite{kipf2017semi, defferrard2016convolutional, frasca2020sign}, various types of multi-head attentions \cite{velickovic2018graph, brody2022attentive}, and graph transformer architectures \cite{shi2021masked}.

In this work, we build upon and improve our preliminary study \cite{giraldo2023higher}.
Specifically, we provide a more comprehensive theoretical explanation, along with an extensive experimental evaluation and in-depth analysis of novel findings.
Additionally, we add a rigorous theoretical study about the stability properties of S2-GNN against possible graph perturbations, which is an important and practical issue in real-world scenarios \cite{gama2020stability,parada2023stability}.
Furthermore, we empirically observe that utilizing smooth-learned graphs \cite{kalofolias2019large,pu2021kernel} improves the performance of both S2-GNN and baseline GNN methods compared to the commonly used $k$-Nearest Neighbors ($k$-NN) technique for graph construction.
Smoother graphs exhibit better results as they promote homophily within the graph.

The main contributions of this work can be summarized as follows:
\begin{enumerate}
    \item We show that the spectrum of higher-order weighted graphs, using the Laplacian matrix, exhibits similar behavior to their sparse counterparts. This finding enables the design of GNN architectures that deliver superior or competitive performance while maintaining efficiency.
    \item We propose S2-GNN, a novel GNN architecture that utilizes a cascade of higher-order sparse filters.
    S2-GNN effectively leverages sparse higher-order operations without excessively increasing complexity.
    \item We rigorously analyze the stability properties of S2-GNN, shedding light on its robustness against possible graph perturbations.
    \item We conduct extensive experimental evaluations on multiple publicly available benchmark datasets and compare S2-GNN against nine state-of-the-art GNNs. Our algorithm outperforms or achieves competitive results compared to prior methods.
\end{enumerate}

The remainder of the paper is organized as follows:
Section \ref{sec:related_works} presents an overview of the related work, while Section \ref{sec:preliminaries} provides an introduction to the preliminary concepts and foundations relevant to this study.
Section \ref{sec:S-SobGNN} provides a detailed explanation of the proposed S2-GNN model and its theoretical stability analysis.
In Section \ref{sec:experiments_results}, we present the experimental framework, results, and discussion.
Finally, Section \ref{sec:conclusions} offers the concluding remarks of the paper.

\section{Related Work}
\label{sec:related_works}

The study of graphs is a well-established field in machine learning and mathematics \cite{chung1997spectral,bronstein2017geometric}.
With the emergence of GNNs in 2014, motivated by the success of CNNs in regular-structured data like images, a wide range of GNN models have been proposed to learn representations of graph-structured data.
Bruna \etal proposed the first modern GNN by extending the convolutional operator of CNNs to graphs \cite{bruna2014spectral}. Defferrard \etal incorporated concepts from GSP to propose localized spectral filtering \cite{defferrard2016convolutional}, while Kipf and Welling approximated the spectral filtering operation to enable efficient Graph Convolution Networks (GCNs) \cite{kipf2017semi}.
These works demonstrated the potential of GNNs and have served as inspiration for the development of various GNN models.

Veličković \etal \cite{velickovic2018graph} proposed Graph Attention Networks (GATs), which employ an attention mechanism to learn different weights for different nodes in the graph.
This enables GATs to effectively model complex relationships between nodes, albeit with increased computational complexity. 
Subsequent works have further explored attention mechanisms in GNNs, such as \cite{kim2021find,brody2022attentive}.

The Simple Graph Convolution (SGC) model was introduced by Wu \etal \cite{wu2019simplifying}, which simplifies the GCN by removing the non-linear activation functions, resulting in improved efficiency but reduced expressivity. 
Later, Chiang \etal \cite{chiang2019cluster} proposed ClusterGCN, a scalable GNN model that leverages graph clustering to enhance efficiency on large-scale graphs.
By partitioning the graph into clusters and applying GNNs to each cluster in parallel, ClusterGCN can learn representations of large-scale graphs.
However, ClusterGCN sacrifices global graph information and the exchange of messages between clusters.
Frasca \etal proposed Scalable Inception Graph Neural Networks (SIGN) \cite{frasca2020sign}, which leverage higher-order relationships in graphs to improve expressivity at the expense of scalability.
SIGN computes the powers of the adjacency matrix, making it unsuitable for large-scale problems due to its complexity.

More recently, graph transformers have emerged as an alternative to GNNs for learning graph-structured data.
Some graph transformers in the literature, similar to GATs, employ attention mechanisms computed in the neighborhood of each node, as shown in \cite{shi2021masked}.
Another approach in graph transformers is to compute attention mechanisms (with or without positional embeddings) by leveraging fully connected graphs \cite{kreuzer2021rethinking}.
However, the latter approach becomes computationally prohibitive for large-scale applications, as the complexity grows quadratically with the number of nodes in the graph.
For a comprehensive review of GNNs, we recommend referring to the survey by Zhang \etal \cite{zhang2022deep}.

Previous studies have focused on improving scalability at the expense of expressivity \cite{wu2019simplifying,chiang2019cluster} or vice versa \cite{defferrard2016convolutional,frasca2020sign}.
Although GAT, its derivatives, and Transformers have been explored, they are not optimal solutions to this problem due to the high computational burden of the attention mechanism and the multi-head function.
As a result, outside the GNN community, the GCN model \cite{kipf2017semi} is the preferred choice due to its simplicity and high efficiency.
In this paper, we propose an accurate and efficient GNN architecture based on higher-order sparse convolutions.
Our algorithm improves accuracy while maintaining low computational and memory footprints.

\section{Preliminaries}
\label{sec:preliminaries}

\subsection{Mathematical Notation}
\label{sec:notation}

In this work, calligraphic letters such as $\mathcal{V}$ denote sets, and $\vert \mathcal{V} \vert$ represents the cardinality of the set. 
Uppercase boldface letters such as $\mathbf{A}$ represent matrices, while lowercase boldface letters such as $\mathbf{x}$ denote vectors. 
$\mathbf{I}$ denotes the identity matrix, and $\mathbf{1}$ represents a vector of ones with appropriate dimensions. 
The superscripts $(\cdot)^{\mathsf{T}}$ and $(\cdot)^{\mathsf{H}}$ correspond to transposition and Hermitian transpose, respectively. 
$\diag(\mathbf{x})$ denotes a diagonal matrix with entries given by the vector $\mathbf{x}=[x_1,x_2,\dots,x_n]^{\mathsf{T}} \in \mathbb{R}^N$.
$\tr(\cdot)$ is the trace of a matrix.
The $\ell_2$-norm of a vector is denoted by $\Vert \cdot \Vert_2$, the Frobenius norm is denoted by $\Vert \cdot \Vert_F$, and the entry-wise norm-$1$ of a matrix is denoted by $\Vert \cdot \Vert_{1,1}$. 
The symbols $\circ$ and $\otimes$ represent the Hadamard and Kronecker products between matrices, respectively.

\subsection{Graph Signals}

A graph is a mathematical entity represented as $G=(\mathcal{V},\mathcal{E})$, where $\mathcal{V}=\{1,\dots,N\}$ is the set of $N$ nodes and ${\mathcal{E}\subseteq \{(i,j)\mid i,j\in \mathcal{V}\;{\textrm {and}}\;i\neq j\}}$ is the set of edges between nodes $i$ and $j$.
The weighted adjacency matrix of the graph is denoted as $\mathbf{A}\in \mathbb{R}^{N\times N}$, where $\mathbf{A}(i,j)=a_{i,j}\in \mathbb{R}_+$ represents the weight of the edge $(i,j)$, and $\mathbf{A}(i,j)=0~\forall~(i,j) \notin \mathcal{E}$.
For undirected graphs, $\mathbf{A}$ is symmetric.
A graph signal is a function $x:\mathcal{V} \to \mathbb{R}$ and is represented as $\mathbf{x} \in \mathbb{R}^N$, where $\mathbf{x}(i)$ is the graph signal evaluated on the $i$th node.
The degree matrix of $G$ is a diagonal matrix given by $\mathbf{D}=\diag(\mathbf{A1})$.
Different definitions exist for the Laplacian matrix, including the combinatorial Laplacian $\mathbf{L=D-A}$, and the symmetric normalized Laplacian $\mathbf{\Delta=I-D}^{-\frac{1}{2}}\mathbf{AD}^{-\frac{1}{2}}$ \cite{ortega2018graph}.
The Laplacian matrix is a positive semi-definite matrix for undirected graphs, with eigenvalues\footnote{$\lambda_N \leq 2$ in the case of the symmetric normalized Laplacian $\mathbf{\Delta}$.} $0=\lambda_1 \leq \lambda_2 \leq \dots \leq \lambda_N$ and corresponding eigenvectors $\{ \mathbf{u}_1,\mathbf{u}_2,\dots,\mathbf{u}_N\}$.
The adjacency and Laplacian matrices, as well as their normalized forms, are examples of the shift operator $\mathbf{S} \in \mathbb{R}^{N \times N}$ in GSP.

The graph Fourier basis of $\mathbf{L}$ is given by the spectral decomposition $\mathbf{L} = \mathbf{U}\mathbf{\Lambda}\mathbf{U}^{\mathsf{T}}$ \cite{ortega2018graph}, where $\mathbf{U}=[\mathbf{u}_1,\mathbf{u}_2,\dots,\mathbf{u}_N]$ and $\mathbf{\Lambda}=\diag([\lambda_1,\lambda_2,\dots,\lambda_N]^{\mathsf{T}})$.
The Graph Fourier Transform (GFT) $\mathbf{\hat{x}}$ of the graph signal $\mathbf{x}$ is defined as $\mathbf{\hat{x}}=\mathbf{U}^{{\mathsf{T}}}\mathbf{x}$, while the inverse GFT is given by $\mathbf{x} = \mathbf{U}\mathbf{\hat{x}}$ \cite{ortega2018graph}.
Although we use the spectral definitions of graphs to establish the relationship between the spectrum of the higher-order Laplacian and its sparse counterpart in this work, the eigenvalue decomposition is not required to implement S2-GNN.

\subsection{Graph Convolutional Filters}

The concept of graph convolutional filters was introduced in the literature of GSP \cite{sandryhaila2013discrete}.
\begin{definition}
    Given a set of parameters $\mathbf{q} = [q_0, q_1, \dots , q_K]^{\mathsf{T}}$ and a graph signal $\mathbf{x} \in \mathbb{R}^N$, a graph convolutional filter of order $K$ is a linear mapping $Q : \mathbb{R}^N \to \mathbb{R}^N$ comprising a linear combination of $K$ shifted signals:
    \begin{equation}
        Q(\mathbf{x}) = \sum_{k=0}^K q_k \mathbf{S}^k \mathbf{x} = \mathbf{Q}(\mathbf{S}) \mathbf{x},
    \end{equation}
    where $\mathbf{Q}(\mathbf{S}) = \sum_{k=0}^K q_k \mathbf{S}^k$ is the $N \times N$ polynomial filtering matrix.
    \label{dfn:graph_filter}
\end{definition}
The output signal $Q(\mathbf{x})$ at node $i$ is a linear combination of signal values located within up to $K$-hops away from node $i$ given by $q_0 \mathbf{x}(i) + q_1 [\mathbf{Sx}](i) + \dots + q_K [\mathbf{S}^K\mathbf{x}](i)$, where $[\mathbf{S}^k\mathbf{x}](i)$ is the $i$th value of the vector $\mathbf{S}^k\mathbf{x}$.


\begin{figure*}
    \centering
    \includegraphics[width=0.95\textwidth]{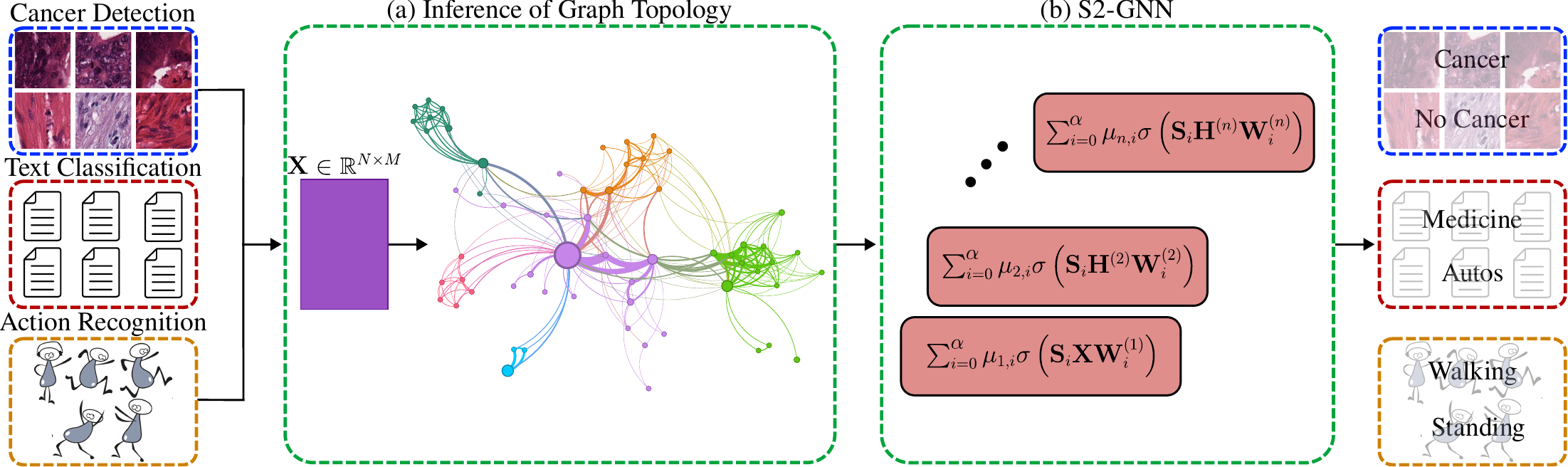}
    \caption{The pipeline of our S2-GNN algorithm. S2-GNN can be used in a broad range of data such as images, text, and videos, among others. However, the step of mapping the original dataset to the data matrix $\mathbf{X}\in \mathbb{R}^{N\times M}$ could be different in each case. Our framework is composed of (a) inference of the graph topology and (b) the S2-GNN architecture.}
    \label{fig:pipeline}
\end{figure*}

\section{Sparse Sobolev Graph Neural Network}
\label{sec:S-SobGNN}

Figure \ref{fig:pipeline} provides an overview of our S2-GNN algorithm.
S2-GNN utilizes either a $k$-Nearest Neighbors ($k$-NN) method or the procedure described in \cite{kalofolias2019large} to establish a graph representation for a specific problem when the structure is unavailable.
In cases where the structure is already known, step a) in Fig. \ref{fig:pipeline} is not required.
Consequently, we employ the proposed GNN architecture to address the primary problem at hand\footnote{In this paper, we focus on the node {and some graph classification tasks}. However, our framework can be easily extended to other problem domains, such as link prediction.}.
Within S2-GNN, a modified sparse Sobolev term serves as the shift operator in the propagation function of the GNN.

\subsection{Sobolev Norm}
\label{sec:sobolev_norm}

The Sobolev norm in GSP has been used as a regularization term to solve problems in 1) video processing \cite{giraldo2020graph}, 2) modeling of infectious diseases \cite{giraldo2020minimization}, and 3) interpolation of graph signals \cite{pesenson2009variational,giraldo2022reconstruction}.
\begin{definition}[Pesenson \cite{pesenson2009variational}]
    \label{dfn:sobolev_norm}
    For fixed parameters $\epsilon \geq 0$ and $\rho\in \mathbb{R}$, the Sobolev norm is defined as $\Vert \mathbf{x} \Vert_{\rho,\epsilon} \triangleq \Vert (\mathbf{L}+\epsilon \mathbf{I})^{\rho/2} \mathbf{x} \Vert$.
    When $\mathbf{L}$ is symmetric, we can rewrite $\Vert \mathbf{x} \Vert_{\rho,\epsilon}^2$ as follows:
    \begin{equation}
        \Vert \mathbf{x} \Vert_{\rho,\epsilon}^2 = \mathbf{x}^{\mathsf{T}}(\mathbf{L}+\epsilon\mathbf{I})^{\rho}\mathbf{x}.
        \label{eqn:sobolev_norm_rewritten}
    \end{equation}
\end{definition}

\noindent We divide the analysis of \eqref{eqn:sobolev_norm_rewritten} into two steps: 1) when $\epsilon=0$, and 2) when $\rho=1$.
For $\epsilon=0$ in (\ref{eqn:sobolev_norm_rewritten}) and considering that $\mathbf{U}$ is orthonormal, we obtain:
\begin{equation}
    \resizebox{0.91\columnwidth}{!}{$\Vert \mathbf{x} \Vert_{\rho,0}^2=\mathbf{x}^{\mathsf{T}}\mathbf{L}^{\rho}\mathbf{x} = \mathbf{x}^{\mathsf{T}}\mathbf{U\Lambda}^{\rho}\mathbf{U}^{\mathsf{T}}\mathbf{x} = \hat{\mathbf{x}}^{\mathsf{T}}\mathbf{\Lambda}^{\rho}\hat{\mathbf{x}}=\sum_{i=1}^N \hat{\mathbf{x}}^2(i) \lambda_i^{\rho}.$}
    \label{eqn:penalized_laplacian}
\end{equation}
Notice that the spectral components $\hat{\mathbf{x}}(i)$ are penalized with powers of the eigenvalues $\lambda_i^{\rho}$ of $\mathbf{L}$.
Since the eigenvalues are ordered in increasing order, the higher frequencies of $\hat{\mathbf{x}}$ are penalized more than the lower frequencies when $\rho=1$, leading to a smooth function in $G$.
For $\rho>1$, the GFT $\hat{\mathbf{x}}$ is penalized with a more diverse set of eigenvalues.
Similarly, we can analyze the adjacency matrix $\mathbf{A}$ using the eigenvalue decomposition $\mathbf{A}^{\rho} = (\mathbf{V\Sigma V}^{\mathsf{H}})^{\rho}=\mathbf{V\Sigma}^\rho\mathbf{V}^{\mathsf{H}}$, where $\mathbf{V}$ is the matrix of eigenvectors, and $\mathbf{\Sigma}$ is the matrix of eigenvalues of $\mathbf{A}$.
For $\mathbf{A}$, the GFT can be defined as $\mathbf{\hat{x}}=\mathbf{V}^{{\mathsf{H}}}\mathbf{x}$.

For the second analysis of (\ref{eqn:sobolev_norm_rewritten}), when $\rho=1$, we have:
\begin{equation}
    \Vert \mathbf{x} \Vert_{1,\epsilon}^2 = \mathbf{x}^{\mathsf{T}}(\mathbf{L}+\epsilon\mathbf{I})\mathbf{x}.
    \label{eqn:sobole_without_rho}
\end{equation}
The term $(\mathbf{L}+\epsilon\mathbf{I})$ in (\ref{eqn:sobole_without_rho}) is associated with a better condition number\footnote{The condition number $\kappa(\mathbf{L})$ associated with the square matrix $\mathbf{L}$ is a measure of how well- or ill-conditioned is the inversion of $\mathbf{L}$.} than using $\mathbf{L}$ alone.
Better condition numbers are associated with faster convergence rates in gradient descent methods, as demonstrated in \cite{giraldo2022reconstruction}.
For the Laplacian matrix $\mathbf{L}$, we know that $\kappa(\mathbf{L}) = \frac{\vert \lambda_{\text{max}}(\mathbf{L})\vert}{\vert \lambda_{\text{min}}(\mathbf{L})\vert} \approx \frac{\lambda_{\text{max}}(\mathbf{L})}{0} \rightarrow \infty$,
where $\kappa(\mathbf{L})$ is the condition number of $\mathbf{L}$, $\lambda_{\text{max}}(\mathbf{L})$ and $\lambda_{\text{min}}(\mathbf{L})$ are the maximum and minimum eigenvalues of $\mathbf{L}$, respectively.
Since $\kappa(\mathbf{L}) \rightarrow \infty$, relying solely on the Laplacian matrix can result in a poorly conditioned problem, particularly when solving an optimization problem \cite{giraldo2022reconstruction}.
On the other hand, for the Sobolev term, we observe that:
\begin{equation}
    \mathbf{L}+\epsilon\mathbf{I} = \mathbf{U\Lambda U}^{\mathsf{T}}+\epsilon\mathbf{I} = \mathbf{U}(\mathbf{\Lambda}+\epsilon\mathbf{I})\mathbf{U}^{\mathsf{T}}.
\end{equation}
Therefore, $\lambda_{\text{min}}(\mathbf{L}+\epsilon\mathbf{I}) = \epsilon$, \ie $\mathbf{L}+\epsilon\mathbf{I}$ is positive definite ($\mathbf{L} +\epsilon\mathbf{I} \succ 0$) for $\epsilon > 0$, and:
\begin{equation}
    \kappa(\mathbf{L}+\epsilon\mathbf{I}) = \frac{\vert \lambda_{\text{max}}(\mathbf{L}+\epsilon\mathbf{I})\vert}{\vert \lambda_{\text{min}}(\mathbf{L}+\epsilon\mathbf{I})\vert} = \frac{\lambda_{\text{max}}(\mathbf{L})+\epsilon}{\epsilon} < \kappa(\mathbf{L});~\forall~\epsilon>0.
    \label{eqn:cond_number_sob_and_laplacian}
\end{equation}
Namely, $\mathbf{L}+\epsilon\mathbf{I}$ has a better condition number than $\mathbf{L}$.
The significance of a better condition number in the context of GNNs might not be immediately apparent since the inverses of the Laplacian or adjacency matrices are not required for performing the propagation rules.
However, several studies have highlighted the adverse effects of poorly conditioned matrices.
For instance, Kipf and Welling \cite{kipf2017semi} employed a renormalization trick ($\mathbf{A+I}$) in their filtering operation to avoid exploding/vanishing gradients.
Similarly, Wu \etal \cite{wu2019simplifying} demonstrated that adding the identity matrix to $\mathbf{A}$ shrinks the graph spectral domain, resulting in a low-pass-type filter.

The previous theoretical analysis shows the benefits of the Sobolev norm in two aspects: 1) the computation of diverse frequencies in (\ref{eqn:penalized_laplacian}), and 2) the better condition number in (\ref{eqn:cond_number_sob_and_laplacian}).

\subsection{Sparse Sobolev Norm}
\label{sec:sparse_sob_norm}

\begin{figure}
    \centering
    \includegraphics[width=\columnwidth]{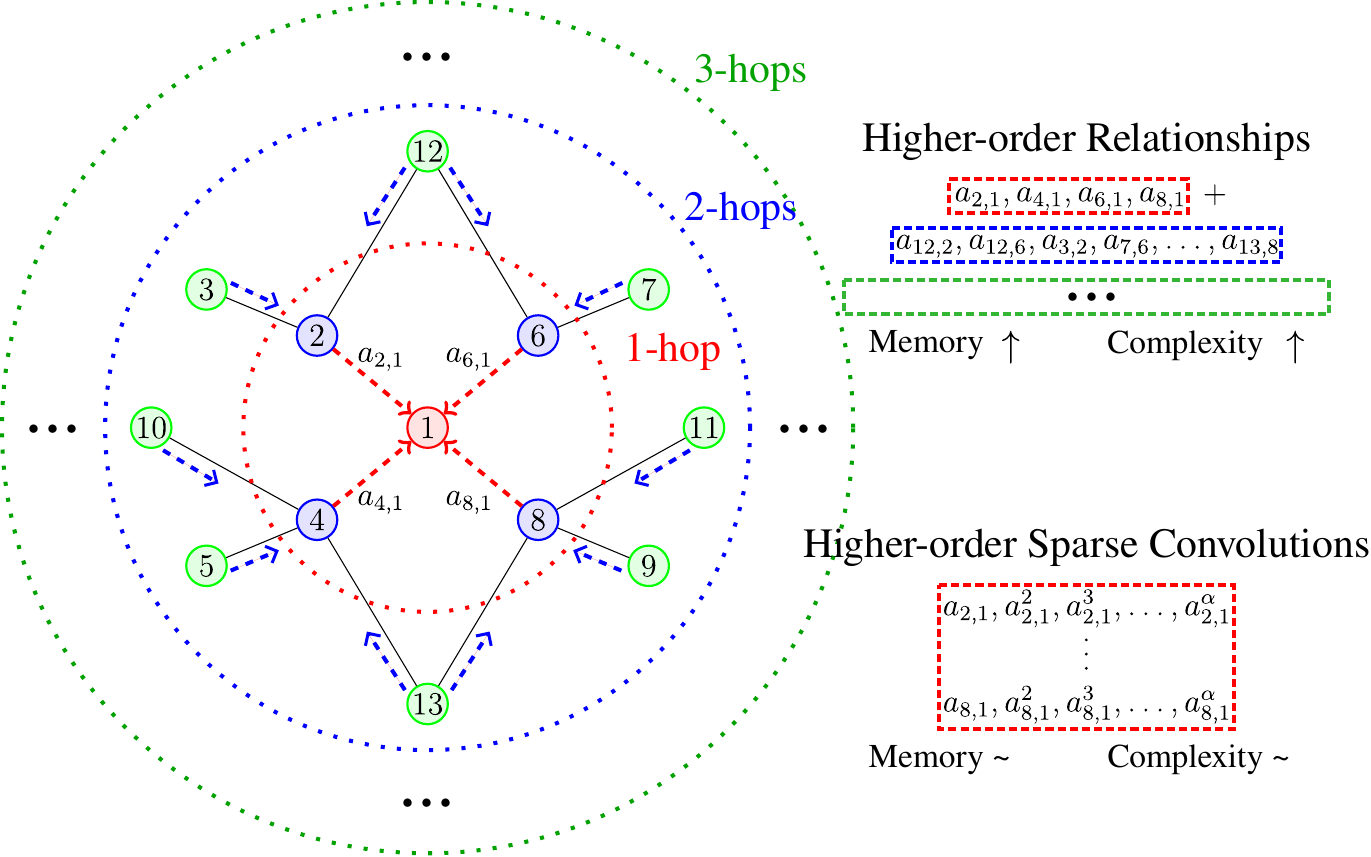}
    \caption{Process of using higher-order relationships in GNNs vs. higher-order sparse convolutions.}
    \label{fig:sparse_operation}
\end{figure}

The use of $\mathbf{L}$, $\mathbf{A}$, or their normalized forms in GNNs is computationally efficient due to their sparsity.
This allows us to perform a small number of sparse matrix operations, which is advantageous for computational efficiency.
However, when considering the Sobolev norm, the term $(\mathbf{L}+\epsilon\mathbf{I})^{\rho}$ can rapidly become a dense matrix for large values of $\rho$, leading to computational and memory issues as illustrated in Fig. \ref{fig:sparse_operation}.
To address this problem, we employ a sparse Sobolev norm that preserves the sparsity level, ensuring computational feasibility and efficient memory usage.
\begin{definition}
    \label{dfn:sparse_sobolev_term}
    Let $\mathbf{L}\in \mathbb{R}^{N\times N}$ be the Laplacian matrix of a graph $G$.
    For fixed parameters $\epsilon \geq 0$ and $\rho \in \mathbb{N}$, the sparse Sobolev term for GNNs is introduced as the $\rho$ Hadamard multiplications of $(\mathbf{L}+\epsilon\mathbf{I})$ (also known as the Hadamard powers) such that:
    \begin{equation}
        (\mathbf{L}+\epsilon\mathbf{I})^{(\rho)}=\underbrace{(\mathbf{L}+\epsilon\mathbf{I}) \circ (\mathbf{L}+\epsilon\mathbf{I}) \circ \dots \circ (\mathbf{L}+\epsilon\mathbf{I})}_{\rho \text{ times}}.
        \label{eqn:sparse_sobolev_term}
    \end{equation}
    For example, $(\mathbf{L}+\epsilon\mathbf{I})^{(2)}=(\mathbf{L}+\epsilon\mathbf{I}) \circ (\mathbf{L}+\epsilon\mathbf{I})$.
    Thus, the sparse Sobolev norm is given by:
    \begin{equation}
        \Vert \mathbf{x} \Vert_{(\rho),\epsilon} \triangleq \Vert (\mathbf{L}+\epsilon \mathbf{I})^{(\rho/2)} \mathbf{x} \Vert.
        \label{eqn:sparse_sobolev_norm}
    \end{equation}
\end{definition}

Let $\langle \mathbf{x}, \mathbf{y} \rangle_{(\rho),\epsilon} = \mathbf{x}^{\mathsf{T}}(\mathbf{L}+\epsilon\mathbf{I})^{(\rho)}\mathbf{y}$ be the inner product between two graph signals $\mathbf{x}$ and $\mathbf{y}$ that induces the associated sparse Sobolev norm.
\begin{theorem}
    The sparse Sobolev norm $\Vert \mathbf{x} \Vert_{(\rho),\epsilon} \triangleq \Vert (\mathbf{L}+\epsilon \mathbf{I})^{(\rho/2)} \mathbf{x} \Vert$ satisfies the basic properties of vector norms for $\epsilon>0$ (for $\epsilon=0$, we obtain a semi-norm).\\
    Proof: See Appendix \ref{app:proof_SSob_norm}.
    \label{trm:SSob_norm}
\end{theorem}

The sparse Sobolev term in (\ref{eqn:sparse_sobolev_term}) possesses the property of maintaining the same sparsity level as $\mathbf{L}+\epsilon\mathbf{I}$ for any value of $\rho$.
It is worth noting that $(\mathbf{L}+\epsilon\mathbf{I})^{\rho}$ is equivalent to the sparse Sobolev term under two conditions: 1) restricting $\rho$ to be in $\mathbb{N}$, and 2) replacing the conventional matrix multiplication with the Hadamard product.
While the theoretical properties of the Sobolev norm in (\ref{eqn:penalized_laplacian}) and (\ref{eqn:cond_number_sob_and_laplacian}) do not directly extend to its sparse counterpart, we can gain some theoretical insights by leveraging the concepts of Kronecker products and the Schur product theorem \cite{horn2012matrix}.

\begin{theorem}
    Let $\mathbf{L}$ be any Laplacian matrix of a graph with eigenvalue decomposition $\mathbf{L}=\mathbf{U\Lambda U}^{\mathsf{T}}$. We observe that:
    \begin{equation}
        \resizebox{0.91\columnwidth}{!}{$\mathbf{L}^{(\rho)} = \mathbf{P}_{(\rho)}^{\mathsf{T}} (\mathbf{U} \otimes \mathbf{U}_{(\rho-1)})(\mathbf{\Lambda} \otimes \mathbf{\Lambda}_{(\rho-1)})(\mathbf{U}^{\mathsf{T}} \otimes \mathbf{U}_{(\rho-1)}^{\mathsf{T}})\mathbf{P}_{(\rho)},$}
        \label{eqn:spectrum_hadamard_product}
    \end{equation}
    where $\mathbf{U}_{(\rho-1)}$ and $\mathbf{\Lambda}_{(\rho-1)}$ are respectively the matrices of eigenvectors and eigenvalues of $\mathbf{L}^{(\rho-1)}$, $\mathbf{P}_{(\rho)} \in \{0,1\}^{N^2 \times N}$ is a partial permutation matrix, and:
    \begin{equation}
        \mathbf{L}^{(2)} = \mathbf{P}_N^{\mathsf{T}} (\mathbf{U} \otimes \mathbf{U})(\mathbf{\Lambda} \otimes \mathbf{\Lambda})(\mathbf{U}^{\mathsf{T}} \otimes \mathbf{U}^{\mathsf{T}})\mathbf{P}_N.
    \end{equation}
    Proof: See Appendix \ref{app:proof_spectrum_hadamard_product}.
    \label{trm:spectrum_hadamard_product}
\end{theorem}


\noindent Theorem \ref{trm:spectrum_hadamard_product} provides the closed-form solution for the spectrum of the Hadamard power of the Laplacian matrix.
It reveals that the spectrum of the Hadamard product is a compressed representation of the Kronecker product of its spectral components.
In the case of the sparse Sobolev term utilized in S2-GNN, denoted by $(\mathbf{L}+\epsilon\mathbf{I})^{(\rho)}$, the spectral components of the graph undergo changes for each value of $\rho$, as shown in \eqref{eqn:spectrum_hadamard_product}.


By applying the Schur product theorem, which states that the Hadamard product between two positive definite matrices is also positive definite, we can analyze the condition number of the Hadamard powers \cite{horn2012matrix}.
We know that $(\mathbf{L}+\epsilon\mathbf{I})^{(\rho)} \succ 0;~\forall~\epsilon > 0$ since $(\mathbf{L}+\epsilon\mathbf{I}) \succ 0;~\forall~\epsilon > 0$, and therefore $\kappa((\mathbf{L}+\epsilon\mathbf{I})^{(\rho)}) < \infty$.
For the adjacency matrix $\mathbf{A}$, its eigenvalues lie within the interval $[-d,d]$, where $d$ is the maximum degree of the vertices in $G$ (see Theorem 8.5 in \cite{nica2018spectral}).
Therefore, we can bound the eigenvalues of $\mathbf{A}$ into the interval $[-1,1]$ by normalizing the adjacency matrix such that $\mathbf{A}_N=\mathbf{D}^{-\frac{1}{2}}\mathbf{AD}^{-\frac{1}{2}}$.
Thus, we can establish that $\mathbf{A}_N+\epsilon \mathbf{I} \succ 0;~\forall~\epsilon > 1$, and subsequently, $(\mathbf{A}_N+\epsilon \mathbf{I})^{(\rho)}\succ 0;~\forall~\epsilon > 1$. 
We can assert that the theoretical developments presented in Section \ref{sec:sobolev_norm} for the sparse Sobolev norm hold to some extent, including a more diverse set of frequencies and a better condition number.
Figure \ref{fig:spectrum_sparse_non_sparse} depicts five normalized eigenvalue penalizations for $\mathbf{L}^{\rho}$ (non-sparse) and $\mathbf{L}^{(\rho)}$ (sparse), revealing that the normalized spectrum of $\mathbf{L}^{\rho}$ and $\mathbf{L}^{(\rho)}$ exhibit a high degree of similarity.

\begin{figure}
    \centering
    \includegraphics[width=\columnwidth]{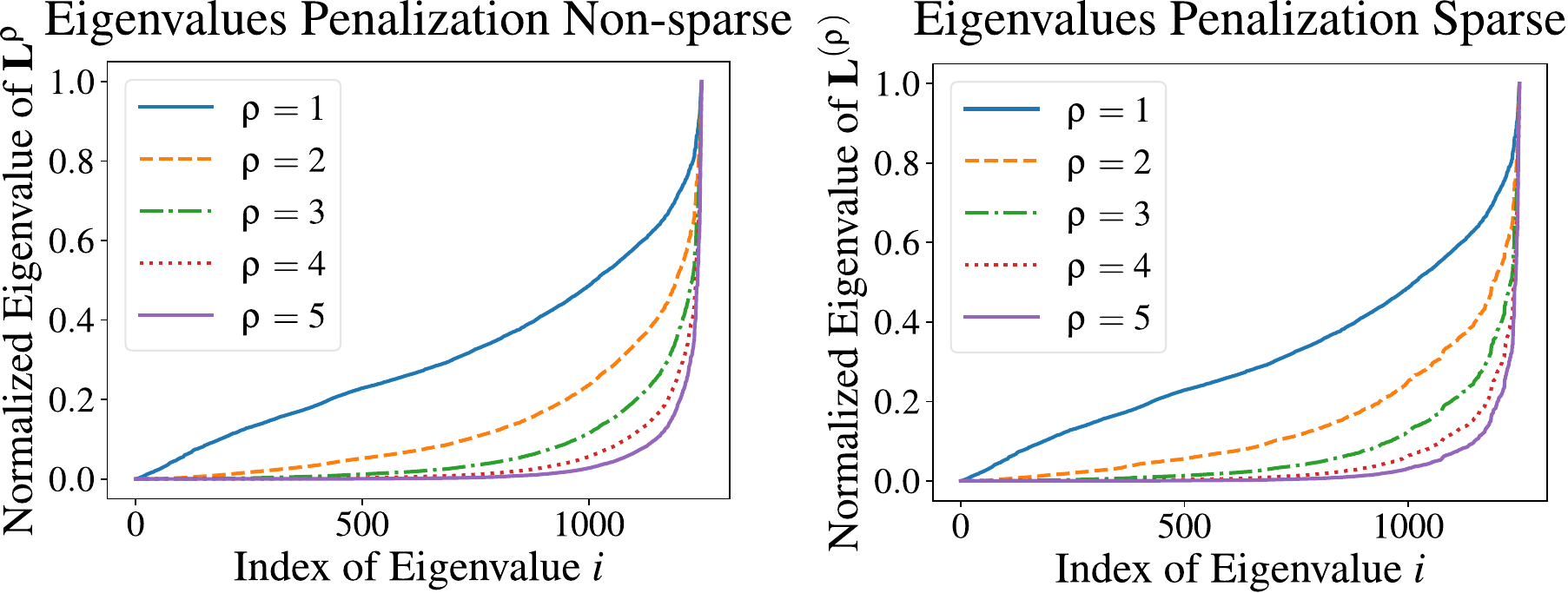}
    \caption{Eigenvalues penalization for the non-sparse and sparse matrix multiplications of the combinatorial Laplacian matrix.}
    \label{fig:spectrum_sparse_non_sparse}
\end{figure}

\subsection{Graph Sparse Sobolev Layer}

In the S2-GNN architecture, we introduce the propagation rule using the graph convolution filter in Definition \ref{dfn:graph_filter}.
Let $\mathbf{S} \in \mathbb{R}^{N \times N}$ be the shift operator of a graph, and let $\mathbf{H}^{(l)} \in \mathbb{R}^{N \times H_l}$ be the output of the previous layer of the GNN (matrix of node embeddings), such that $\mathbf{H}^{(1)}=\mathbf{X} \in \mathbb{R}^{N \times M}$ is an $M$-dimensional graph signal (matrix of features). The propagation rule of the $\rho$th branch of S2-GNN is defined as:
\begin{equation}
    \mathbf{B}_{\rho}^{(l+1)} = \sigma \left( \mathbf{S}_{\rho} {\mathbf{H}}^{(l)}\mathbf{W}_{\rho}^{(l)} \right),
    \label{eqn:propagation_rule_sobolev}
\end{equation}
where $\sigma(\cdot)$ is some activation function, $\mathbf{S}_{\rho} = \mathbf{S}^{(\rho)}$, and $\mathbf{W}_{\rho}^{(l)} \in \mathbb{R}^{H_l \times H_ {l+1}}$ is the matrix of trainable weights in the $\rho$th branch of the layer $l$.
The matrix $\mathbf{W}_{\rho}^{(l)}$ corresponds to a Multi-Layer Perceptron (MLP) with $H_ {l+1}$ hidden units. In practice, we typically include a bias term in the MLP; however, we have omitted it here for simplicity.
Notice that \eqref{eqn:propagation_rule_sobolev} is an approximation of the graph convolutional filter in Definition \ref{dfn:graph_filter} when substituting the regular matrix power by the Hadamard power, $q_i = 1$ for $i=\rho$, $q_i = 0;~\forall~i \neq \rho$, the input signal is a $H_l$-dimensional graph signal, and adding an MLP at the end of the filtering process.

The S2-GNN architecture consists of multiple layers, where each layer involves computing a cascade of $\alpha+1$ propagation rules \eqref{eqn:propagation_rule_sobolev} with increasing Hadamard powers.
Finally, the outputs from all branches are combined using some fusion mechanism, such as a linear combination layer or an MLP fusion technique.

\noindent \textbf{Linear combination layer:} Let $\mathbf{B}_{\rho}^{(l+1)} \in \mathbb{R}^{N \times H_{l+1}}$ be the output of the $\rho$th branch in the layer $l$ of S2-GNN.
The output of the $l$th layer of S2-GNN using the linear combination layer is given by:
\begin{equation}
    \label{eqn:linear_combination_layer}
    \mathbf{H}^{(l+1)} = \sum_{i=0}^{\alpha} \mu_{l,i} \mathbf{B}_{i}^{(l+1)},
\end{equation}
where $\mu_{l,i}$ is the learnable weight associated with $\mathbf{B}_{i}^{(l+1)}$, and we consider $\mathbf{B}_{0}^{(l+1)} = \sigma \left( \mathbf{I} {\mathbf{H}}^{(l)}\mathbf{W}_{0}^{(l)} \right) = \sigma \left( {\mathbf{H}}^{(l)}\mathbf{W}_{0}^{(l)} \right)$.


\noindent \textbf{MLP fusion:} Let $\mathbf{B}_{\rho}^{(l+1)} \in \mathbb{R}^{N \times H_{l+1}}$ be the output of the $\rho$th branch in the layer $l$ of S2-GNN.
The output of the $l$th layer of S2-GNN using the MLP fusion is given by:
\begin{equation}
    \label{eqn:MLP_fusion_layer}
    \mathbf{H}^{(l+1)} = \left[\mathbf{B}_{0}^{(l+1)},\mathbf{B}_{1}^{(l+1)}, \dots, \mathbf{B}_{\alpha}^{(l+1)} \right] \mathbf{W}_{\text{MLP}}^{(l)},
\end{equation}
where $\left[\mathbf{B}_{0}^{(l+1)},\mathbf{B}_{1}^{(l+1)}, \dots, \mathbf{B}_{\alpha}^{(l+1)} \right]$ is a concatenation of the matrices $\mathbf{B}_{i}^{(l+1)}~\forall~0 \leq i \leq \alpha$, $\mathbf{B}_{0}^{(l+1)} = \sigma \left( {\mathbf{H}}^{(l)}\mathbf{W}_{0}^{(l)} \right)$, and $\mathbf{W}_{\text{MLP}}^{(l)} \in \mathbb{R}^{(\alpha+1) H_{l+1} \times H_{l+1}}$ is the matrix of learnable weights to combine the $\alpha+1$ branches.


It is worth noting that a layer of S2-GNN, as described in \eqref{eqn:linear_combination_layer} or \eqref{eqn:MLP_fusion_layer}, can be viewed as a mapping $H: \mathbb{R}^{N \times H_{l}} \to \mathbb{R}^{N \times H_{l+1}}$.
In practice, the choice of the fusion mechanism is determined by a hyperparameter optimization algorithm that selects the most suitable mechanism for a given dataset.
{Other techniques, like attention-based fusion methods \cite{peng2021attention} could also be applied into our model.
However, such techniques could increase the overall complexity of S2-GNN due to the computation of the attention mechanism.}

\subsubsection{Re-normalization Trick and Sobolev Term}

In practice, we utilize a normalized Sobolev version of the adjacency matrix as the shift operator for computing our propagation rule.
Specifically, we set $\mathbf{S}_{(\rho)}=\bar{\mathbf{D}}_{\rho}^{-\frac{1}{2}}\bar{\mathbf{A}}_{\rho}\bar{\mathbf{D}}_{\rho}^{-\frac{1}{2}}$, where $\bar{\mathbf{A}}_{\rho}=\left( \mathbf{A}+\epsilon\mathbf{I} \right)^{(\rho)}$ represents the $\rho$th sparse Sobolev term of $\mathbf{A}$ and $\bar{\mathbf{D}}_{\rho}$ denotes the degree matrix of $\bar{\mathbf{A}}_{\rho}$. 
This operation can be interpreted from a message-passing perspective, where we aggregate the embeddings from the 1-hop neighborhood with weights proportional to the edge attributes, node degrees, and the Hadamard power involved in that operation.
Unlike the GCN operation, our approach incorporates the Hadamard power term.
Our operation does not compute attention coefficients like GAT or graph Transformers, but it does learn the importance of each branch through the fusion layer, either by using a linear combination or an MLP operation.

\begin{figure*}
    \centering
    \includegraphics[width=0.9\textwidth]{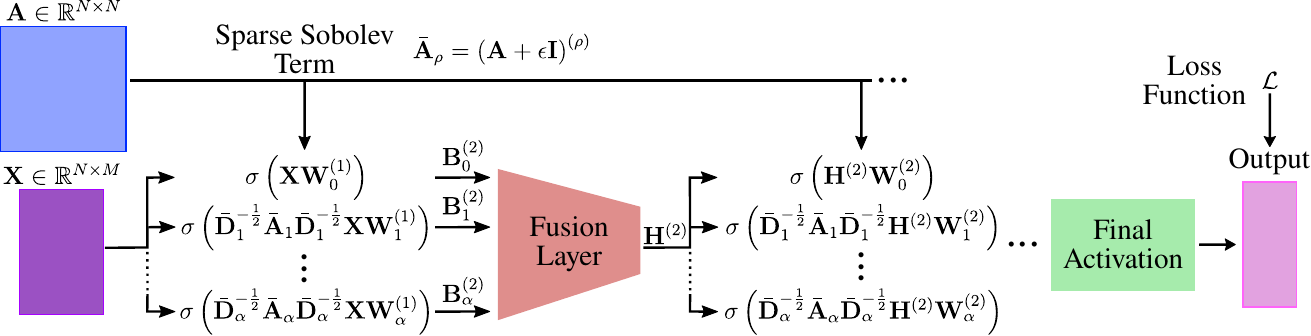}
    \caption{Basic configuration of our S2-GNN architecture with $n$ layers and $\alpha$ branches per layer.}
    \label{fig:S-SobGNN_Architecture}
\end{figure*}

Figure \ref{fig:S-SobGNN_Architecture} illustrates the basic configuration of S2-GNN, serving as the core for constructing more complex architectures based on our filtering operation.
Additional components such as batch normalization and residual connections can be added.
Notice that our graph convolution is efficiently computed since the term $\bar{\mathbf{D}}_{\rho}^{-\frac{1}{2}}\bar{\mathbf{A}}_{\rho}\bar{\mathbf{D}}_{\rho}^{-\frac{1}{2}};~\forall~\rho \in \{1,2,\dots,\alpha\}$ remains the same across all layers (allowing for offline computation), and these terms remain sparse for any value of $\rho$ (given the sparsity of $\mathbf{A}$).
S2-GNN can employ conventional activation functions like $\relu$ or $\softmax$, as well as typical loss functions $\mathcal{L}$ suited to the specific task.
The basic architecture of S2-GNN is defined by the number of filters $\alpha$ in each layer, the parameter $\epsilon$, the number of hidden units in each $\mathbf{W}_{\rho}^{(l)}$, and the number of layers $n$.
When constructing weighted graphs with Gaussian kernels, the edge weights fall within the interval $[0,1]$.
As a result, large values of $\rho$ could lead to $\bar{\mathbf{A}}_{\rho}=\mathbf{0}$, and the diagonal elements of $\bar{\mathbf{D}}_{\rho}^{-\frac{1}{2}}$ may become $\infty$.
Similarly, when $\alpha$ is very large, the resulting architectures become wide with a significant number of parameters, making it desirable to maintain a reasonable value for $\alpha$.
In our experiments, we evaluate S2-GNN architectures up to $\alpha=6$.

\subsection{Relationship with Graph Convolutional Network and Computational Complexity}

Kipf and Welling \cite{kipf2017semi} proposed one of the most successful yet simple GNNs, called Graph Convolutional Networks (GCNs), using the following propagation rule:
\begin{equation}
    \bar{\mathbf{H}}^{(l+1)}=\sigma(\tilde{\mathbf{D}}^{-\frac{1}{2}}\tilde{\mathbf{A}}\tilde{\mathbf{D}}^{-\frac{1}{2}}\bar{\mathbf{H}}^{(l)}\mathbf{W}^{(l)}).
    \label{eqn:propagation_rule}
\end{equation}
Here $\tilde{\mathbf{A}}=\mathbf{A}+\mathbf{I}$, $\tilde{\mathbf{D}}$ represents the degree matrix of $\tilde{\mathbf{A}}$, $\bar{\mathbf{H}}^{(l)} \in \mathbb{R}^{N \times H_l}$ denotes the output of the previous layer $l$, where $\bar{\mathbf{H}}^{(1)}=\mathbf{X} \in \mathbb{R}^{N \times M}$ is matrix of input features, and $\mathbf{W}^{(l)} \in \mathbb{R}^{H_l \times H_{l+1}}$ is the matrix of trainable weights in layer $l$.
The motivation behind the propagation rule in (\ref{eqn:propagation_rule}) is based on the first-order approximation of localized spectral filters on graphs \cite{defferrard2016convolutional}.
We can observe that GCN is an approximation of the graph convolutional filter in Definition \ref{dfn:graph_filter}, where $q_1 = 1$, $q_i = 0;~\forall~i \neq 1$, the input signal is an $H_l$-dimensional graph signal, and an MLP is added at the end of the filtering process.
Notice that our S2-GNN layer is a generalization of the GCN layer when only the branch $\rho=1$ is used in \eqref{eqn:propagation_rule_sobolev} without the fusion layer, and the shift operator is the normalized adjacency matrix.

The computational complexity of a GCN with $n$ layers is $\mathcal{O}(n \vert \mathcal{E} \vert)$.
The GCN model is widely used in practice due to its simplicity, good performance, and low computational complexity.
For S2-GNN, the computational complexity of each branch remains $\mathcal{O}(\vert \mathcal{E} \vert)$ since our sparse Sobolev term maintains the same sparsity level in the graph representation for any value of $\rho$.
When considering $\alpha$ branches, the complexity becomes $\mathcal{O}(\alpha \vert \mathcal{E} \vert + 2\alpha)$ per layer, taking into account the linear combination layer and the branch for $\rho=0$.
With $n$ layers, the computational complexity of our architecture is $\mathcal{O}\left(n(\alpha \vert \mathcal{E} \vert + 2\alpha)\right) = \mathcal{O}(n \alpha \vert \mathcal{E} \vert + 2n\alpha)$.
Therefore, S2-GNN achieves a computational complexity similar to GCN while delivering better performance on several benchmark datasets.
It is worth noting that the exact complexity of both methods, S2-GNN and GCN, also depends on factors such as the feature dimension, the hidden units, and the number of nodes in the graph, which have been omitted for simplicity and clarity.

\subsection{Theoretical Stability Analysis}
\label{sec:stability_analysis}

Here, we aim to study the effect of graph perturbations on the prediction performance of the proposed network.
Similarly, we analyze the stability \cite{gama2020stability} of S2-GNN against the existence of erroneous edges. 
Particularly, we consider the deviation of the adjacency matrix $\mathbf{A}$ with the error (or perturbation) matrix $\mathbf{E}$ as follows:
\begin{gather}
    \label{pert_model}
    \hat{\mathbf{A}} = \mathbf{A} + \mathbf{E};~\mathbf{E}\in\mathcal{P},\\
    \nonumber
    \resizebox{\columnwidth}{!}{$\mathcal{P} =\left\{\mathbf{E}\in\mathbb{R}^{N\times N}~\vert~\mathbf{E}=\mathbf{E}^\top,\vert e_{i,j}\vert \le a_{i,j};~\forall~ i,j,e_{i,i}=0;~\forall i \right\}.$}
\end{gather}
Note that the symmetry of the error matrix $\mathbf{E}$ stems from the underlying graph being undirected. 
The conditions on $\mathcal{P}$ imply that: 1) the perturbations cannot cause negative edge weights, and 2) we do not consider edge perturbations causing self-loops.

For the sake of simplicity and without loss of generality, we consider the true and perturbed output of a simplified \cite{wu2019simplifying} version of the S2-GNN (SS2-GNN), which can be stated as:
\begin{equation}
\begin{split}
&\mathbf{Y}=\sigma\left(\mathbf{L}^{(\rho)}\mathbf{X\mathbf{W}}\right),\\
&\hat{\mathbf{Y}}=\sigma\left(\hat{\mathbf{L}}^{(\rho)}\mathbf{X}\hat{\mathbf{W}}\right),
\end{split}
\label{eqn:perturbations}
\end{equation}
where $\hat{\mathbf{L}}$ is the perturbed Laplacian matrix. The next theorem obtains an upper bound on the distance between the true and perturbed outputs (\ie $d(\mathbf{Y},\hat{\mathbf{Y}})=\|\mathbf{Y}-\hat{\mathbf{Y}}\|$) for the SS2-GNN in the case of having true and perturbed versions of the graph and graph filter coefficients. 
First, we need to mention the definition of being Lipschitz for a general function $\sigma(\cdot)$ in $\mathbb{R}$.
\begin{definition}
\label{Lipschitz}
We call a function $\sigma(\cdot)$ as Lipschitz if there exists a positive constant $\phi$ 
such that:
\begin{equation}
\forall~x_1, x_2 \in \mathbb{R}: \vert\sigma(x_1)-\sigma(x_2)\vert\le \phi\vert x_1-x_2 \vert.
\end{equation}
\end{definition}

\begin{theorem}
\label{dYY}
With the assumptions of the non-linear function $\sigma(\cdot)$ being Lipschitz, the node feature matrix $\mathbf{X}$ is column-wisely normalized as $\|\mathbf{X}\|_F=\sqrt{\upsilon}$, $\|\mathbf{W}-\hat{\mathbf{W}}\|\le\delta_\mathbf{W}$, $\|\mathbf{E}\|\le\xi$, and $\|\mathbf{D}_E\|\le\xi_D$, where $\mathbf{D}_E$ is the degree matrix of $\mathbf{E}$, the difference between the true and perturbed outputs of an SS2-GNN is bounded by:
\begin{equation}
\label{d_yhat}
\begin{split}
&d(\mathbf{Y},\hat{\mathbf{Y}})\\
&=\left\|\sigma\left(\mathbf{L}^{(\rho)}\mathbf{X\mathbf{W}}\right)-\sigma\left(\hat{\mathbf{L}}^{(\rho)}\mathbf{X}\hat{\mathbf{W}}\right)\right\|\\
&\le\sqrt{\upsilon}\left[(\hat{d}^{max}+\epsilon)^\rho+\|\hat{\mathbf{A}}^{(\rho)}\|\right]\delta_{\mathbf{W}} + \mathcal{O}(\xi^2)+\mathcal{O}(\xi^2_D)\\
&+\rho\sqrt{\upsilon}\left[\eta \|\mathbf{E}\|+(d^{max}+\epsilon)^{\rho-1} d^{max}_{e}\right] \|\mathbf{W}\|,
\end{split}
\end{equation}
where $\eta \triangleq \min\left( r_1(\mathbf{A}^{(\rho-1)}),c_1(\mathbf{A}^{(\rho-1)})\right)$, $r_1(\cdot)$ and $c_1(\cdot)$ denote the row and column of the involved matrix with highest norm-2 among the other rows and columns, respectively.\\
Proof: See Appendix \ref{AppC}.
\end{theorem}

The obtained bound in \eqref{d_yhat} reveals interesting aspects of the S2-GNN stability against possible perturbations.
Precisely, this bound relies on the Hadamard order ($\rho$), which implies that a high degree of Hadamard multiplications could make the network more susceptible to perturbations. 
Besides, the absolute dependency of this bound on $\|\mathbf{E}\|$, $\xi$, and $\xi_D$, which are respectively representations of the perturbation (noise) power, depicts the effect of perturbation amplitudes on the network stability.
The existence of $\|\mathbf{W}\|$ and $\delta_\mathbf{W}$ emphasizes the choice of the network size and selection of optimization approach (\eg stochastic gradient descent) for reaching the suboptimal solution for $\|\mathbf{W}\|$, which has a direct impact on the distance (\ie $\delta_\mathbf{W}$) from the true values of $\mathbf{W}$.
Also, as can be seen in (\ref{d_yhat}), an unnecessarily high value of $\epsilon$ could have considerably negative effects on the network robustness, especially in the case of high Hadamard order ($\rho$).
Finally, due to the existence of $\eta$ and more importantly, $d^{max}$, in the upper bound (\ref{d_yhat}), the smaller values for $d^{max}$ (\ie sparse graphs) help in stability and robustness of the network.
Therefore, the proposed network can effectively benefit from the sparse underlying structures, which are widely found in real-world applications.
{We empirically validate the stability analysis of S2-GNN in Appendix \ref{app:experiments_stability_analysis}.}

\section{Experimental Evaluation}
\label{sec:experiments_results}

S2-GNN is compared to nine state-of-the-art methods in GNNs.
These methods include Chebyshev graph convolutions (Cheby) \cite{defferrard2016convolutional}, Graph Convolutional Networks (GCN) \cite{kipf2017semi}, Graph Attention Networks (GAT) \cite{velickovic2018graph}, Simple Graph Convolution (SGC) \cite{wu2019simplifying}, ClusterGCN \cite{chiang2019cluster}, Scalable Inception Graph Neural network (SIGN) \cite{frasca2020sign}, SuperGAT \cite{kim2021find}, a graph Transformer \cite{shi2021masked}, and GAT version 2 (GATv2) \cite{brody2022attentive}.

\subsection{Datasets}

We evaluate the performance of S2-GNN on multiple datasets, which include tissue phenotyping in colon cancer histology images \cite{kather2016multi}, text classification of news \cite{lang1995newsweeder}, activity recognition using sensors \cite{anguita2013public}, recognition of spoken letters \cite{fanty1991spoken}, as well as the commonly used node classification benchmarks Cora \cite{mccallum2000automating}, Citeseer \cite{sen2008collective}, Pubmed \cite{namata2012query}, and OGBN-proteins \cite{hu2020open}.
These diverse datasets provide a comprehensive evaluation of S2-GNN's performance across different domains.

\begin{table*}[]
\centering
\caption{Statistics of the datasets in the experimental framework of this work.}
\label{tbl:statistics_datasets}
\resizebox{0.9\textwidth}{!}{
\begin{threeparttable}
\begin{tabular}{r|ccccccccc}
\toprule
         & \textbf{Cancer-B} & \textbf{Cancer-M} & \textbf{20News} & \textbf{HAR} & \textbf{Isolet} & \textbf{Cora} & \textbf{Citeseer} & \textbf{Pubmed} &
         \textbf{OGBN-proteins}
         \\
\midrule
Nodes & $1,250$ & $5,000$ & $3,000$ & $3,000$ & $3,000$ & $2,485$ & $2,120$ & $19,717$ & $30,000$ \\
Edges$^*$ & $25,734$ & $101,439$ & $62,399$ & $87,032$ & $43,330$ & $205,928$ & $183,572$ & $1,723,360$ & $13,251,798$ \\
Edges$^\dagger$ & $19,163$ & $81,462$ & $47,496$ & $61,400$ & $31,130$ & NA & NA & NA & NA\\
Features & $38$ &  $38$ & $2,000$ & $561$ & $617$ & $1,433$ & $3,703$ & $500$ & $8$ \\
Classes &  $2$ &  $8$ & $10$ & $6$ & $26$ & $7$ & $6$ & $3$ & $112$\\
Undirected & \checkmark & \checkmark & \checkmark & \checkmark & \checkmark & \checkmark & \checkmark & \checkmark & \checkmark\\
\bottomrule
\end{tabular}
\begin{tablenotes}\footnotesize
\item \scriptsize $^*$ This number of edges corresponds to the $k$-NN method for the constructed graphs, and the outputs of the GDC method for the citation networks.
\item \scriptsize $^\dagger$ This number of edges corresponds to the method in \cite{kalofolias2019large} for the constructed graphs.
\end{tablenotes}
\end{threeparttable}
}
\end{table*}

\noindent \textbf{The colorectal cancer histology} dataset comes from ten anonymized tissue slides from the University Medical Center Mannheim \cite{kather2016multi}.
The dataset consists of eight classes: tumor epithelium, simple stroma, complex stroma, immune cells, debris, normal mucosal glands, adipose tissue, and background (no tissue).
This work does not explore the feature representation, and only uses Local Binary Patterns (LBP) to represent each image.
We explore two partitions of this dataset: 1) Cancer Binary (Cancer-B) considers only tumor epithelium and simple stroma, and 2) Cancer Multi-class (Cancer-M) tests all classes.

\noindent \textbf{The 20 Newsgroups} dataset (20News) contains approximately $20,000$ news articles categorized into $20$ classes \cite{lang1995newsweeder}.
In this work, we use a subset of ten categories: computer graphics, windows operative system, IBM hardware, MAC hardware, autos, motorcycles, cryptography, electronics, medicine, and space.
Each news document is represented using the Term Frequency - Inverse Document Frequency (TFIDF) as described in \cite{gadde2014active}.

\noindent \textbf{The Human Activity Recognition} dataset (HAR) comprises recordings of $30$ individuals engaged in various activities of daily living \cite{anguita2013public}.
The participants wore waist-mounted smartphones with embedded sensors, and the dataset consists of six activity classes: standing, sitting, laying down, walking, walking downstairs, and walking upstairs.
The recordings are represented as in \cite{anguita2013public}.

\noindent \textbf{The Isolated Letter recognition} dataset (Isolet) includes sound recordings from the $26$ letters of the English alphabet spoken by $30$ speakers.
Isolet contains $6,238$ samples represented as $617$-dimensional feature vectors.
The feature representation method is explained in \cite{fanty1991spoken}.

\noindent \textbf{The Cora, Citeseer, and Pubmed} datasets are well-known citation networks.
The datasets consist of sparse bag-of-words feature vectors for each document (node) and a list of citation links between documents.
In this work, we treat the citation links as undirected edges as in \cite{kipf2017semi}.
To fulfill the requirement of weighted graphs, as mentioned in Section \ref{sec:preliminaries}, we preprocess the \textit{unweighted} citation networks using the Graph Diffusion Convolution (GDC) method \cite{klicpera2019diffusion}.

\noindent \textbf{The Open Graph Benchmark Proteins} dataset (OGBN-Proteins) is a comprehensive graph-based dataset designed for protein-protein interaction networks \cite{hu2020open}. 
The dataset features more than $132,000$ nodes (proteins) and $39$ million edges, reflecting a diverse range of biological relationships. 
Each node in the graph is associated with $8$-dimensional feature vectors, capturing essential biological characteristics of the proteins. 
The primary task of the dataset is to predict protein functions, with each protein potentially linked to one or more of the $112$ function categories.
Due to computational constraints, the scope of the data used in the experiments is reduced to approximately $25\%$ of the full dataset, which equates to around $30,000$ nodes.
This subset selection ensures manageable computational demands while still providing a representative sample for analysis.

\subsection{Experiments}
\label{sec:experiments}

For the Cancer-B, Cancer-M, 20News, HAR, and Isolet datasets, we generate graphs using two methods: $k$-Nearest Neighbors ($k$-NN) and the algorithm described in \cite{kalofolias2019large}.
In our results, we refer to the graphs generated by the latter method as \textit{learned graphs}.
Regarding Cora, Citeseer, and Pubmed datasets, we consider the largest connected component of the graph following common practices \cite{klicpera2019diffusion}. OGBN-proteins is a dataset that comprises a single weighted graph, where edges represent the confidence level of interactions between proteins, eliminating the necessity for preprocessing steps like GDC. This simplifies the data preparation phase and allows for a more direct application of the model.
Table \ref{tbl:statistics_datasets} presents a comprehensive summary of the statistics for all the graphs used in the study.
To optimize the hyperparameters of citation networks, we employ a random search procedure that maximizes the average accuracy on the validation sets. 
For OGBN-proteins, we empirically obtain the hyperparameters, focusing on the critical importance of learning rate and hidden channels over other hyperparameters.
We split the datasets into training, validation, and testing sets.
Initially, we divide the data into a development set and a test set, ensuring that the test set is not used during the hyperparameter optimization process.
Following the experimental framework of \cite{klicpera2019diffusion}, we optimize the hyperparameters for each dataset-preprocessing combination separately using a random search over five data splits, with multiple repetitions for enhanced robustness.
For the Cora, Citeseer, and Pubmed datasets, the development set consists of $1,500$ nodes, while the remaining nodes are used for testing.
Similarly, the train set includes $20$ nodes from each class, and the remaining nodes are utilized for validation.
As for the other datasets, we employ a 10/45/45 split, meaning that 10\% of the nodes are assigned for training, 45\% for validation, and 45\% for testing.
Finally, we report average accuracies on the test set accompanied by 95$\%$ confidence intervals calculated via bootstrapping with $1,000$ samples, with each model being tested using $50$ different seeds to ensure robustness and reliability of the results.

\begin{table*}[]
\centering
\caption{Accuracy (in \%) for the state-of-the-art methods and our S2-GNN algorithm in five datasets, inferring the graphs with $k$-NN or the method for learning smooth graphs in \cite{kalofolias2019large}.}
\label{tbl:results_constructed_graphs}
\resizebox{\textwidth}{!}{
\begin{threeparttable}
\begin{tabular}{r|ccccc|ccccc}
\toprule
 & \multicolumn{5}{c|}{\textbf{$k$-NN Graphs}} & \multicolumn{5}{c}{\textbf{Learned Graphs}} \\
\multirow{-2}{*}{\textbf{Model}} & \textbf{Cancer-B} & \textbf{Cancer-M} & \textbf{20News} & \textbf{HAR} & \textbf{Isolet} & \textbf{Cancer-B} & \textbf{Cancer-M} & \textbf{20News} & \textbf{HAR} & \textbf{Isolet} \\
\midrule
Cheby \cite{defferrard2016convolutional} & $82.55_{\pm 4.95}$ & $59.52_{\pm 1.42}$ & $69.53_{\pm 1.47}$ & $85.64_{\pm 2.92}$ & $71.52_{\pm 1.42}$ & $87.64_{\pm 3.69}$ & $57.50_{\pm 2.12}$ & \color{blue} $\underline{\textit{72.49}}_{\pm 0.82}$ & $87.07_{\pm 3.57}$ & $59.14_{\pm 2.50}$ \\
GCN \cite{kipf2017semi} & $71.31_{\pm 5.72}$ & $50.87_{\pm 2.39}$ & $63.08_{\pm 2.36}$ & $64.28_{\pm 4.90}$ & $63.61_{\pm 1.76}$ & $80.50_{\pm 4.77}$ & $52.45_{\pm 3.00}$ & $69.83_{\pm 1.25}$ & $82.85_{\pm 3.00}$ & $71.27_{\pm 2.01}$ \\
GAT \cite{velickovic2018graph} & $68.67_{\pm 5.24}$ & $48.87_{\pm 1.94}$ & $48.90_{\pm 2.17}$ & $70.13_{\pm 3.65}$ & $58.93_{\pm 2.03}$ & $77.69_{\pm 4.52}$ & $51.23_{\pm 2.04}$ & $54.30_{\pm 2.13}$ & $75.45_{\pm 3.97}$ & $58.78_{\pm 1.75}$ \\
SGC \cite{wu2019simplifying} & $72.00_{\pm 4.56}$ & $42.22_{\pm 1.88}$ & $54.24_{\pm 2.40}$ & $39.39_{\pm 2.72}$ & $39.90_{\pm 2.05}$ & $70.82_{\pm 4.87}$ & $43.97_{\pm 1.97}$ & $72.13_{\pm 1.31}$ & $38.22_{\pm 2.71}$ & $43.21_{\pm 1.85}$ \\
ClusterGCN \cite{chiang2019cluster} & $83.03_{\pm 4.43}$ & $56.85_{\pm 1.25}$ & $67.95_{\pm 1.55}$ & $72.67_{\pm 4.70}$ & $59.05_{\pm 3.20}$ & $85.63_{\pm 4.26}$ & $58.58_{\pm 1.76}$ & $62.12_{\pm 1.59}$ & $65.98_{\pm 4.45}$ & $69.54_{\pm 2.37}$ \\
SIGN \cite{frasca2020sign} & \color{blue} $\underline{\textit{89.67}}_{\pm 0.39}$ & \color{blue} $\underline{\textit{62.45}}_{\pm 0.17}$ & \color{red} $\textbf{72.34}_{\pm 0.25}$ & \color{blue} $\underline{\textit{91.58}}_{\pm 0.24}$ & \color{blue} $\underline{\textit{84.11}}_{\pm 0.24}$ & \color{blue} $\underline{\textit{90.30}}_{\pm 0.38}$ & \color{blue} $\underline{\textit{64.99}}_{\pm 0.20}$ & \color{red} $\textbf{74.10}_{\pm 0.30}$ & \color{blue} $\underline{\textit{92.88}}_{\pm 0.23}$ & \color{blue} $\underline{\textit{84.17}}_{\pm 0.27}$ \\
SuperGAT \cite{kim2021find} & $69.54_{\pm 5.09}$ & $40.04_{\pm 1.84}$ & $54.40_{\pm 1.95}$ & $72.66_{\pm 3.76}$ & $60.57_{\pm 2.17}$ & $78.15_{\pm 4.42}$ & $52.73_{\pm 2.07}$ & $59.53_{\pm 1.96}$ & $74.93_{\pm 4.00}$ & $61.49_{\pm 1.68}$ \\
Transformer \cite{shi2021masked} & $74.23_{\pm 5.65}$ & $54.20_{\pm 1.70}$ & $59.06_{\pm 2.75}$ & $79.66_{\pm 3.34}$ & $65.68_{\pm 2.42}$ & $79.23_{\pm 4.83}$ & $52.07_{\pm 2.21}$ & $56.19_{\pm 2.83}$ & $78.30_{\pm 2.81}$ & $65.36_{\pm 2.18}$ \\
GATv2 \cite{brody2022attentive} & $62.89_{\pm 5.21}$ & $47.15_{\pm 2.07}$ & $46.56_{\pm 2.04}$ & $73.28_{\pm 4.02}$ & $64.59_{\pm 1.40}$ & $79.28_{\pm 4.01}$ & $45.80_{\pm 1.78}$ & $51.17_{\pm 2.14}$ & $75.09_{\pm 4.15}$ & $66.46_{\pm 1.87}$ \\
\midrule
S2-GNN (ours) & \color{red} $\textbf{93.39}_{\pm 0.41}$ & \color{red} $\textbf{68.48}_{\pm 0.28}$ & \color{blue} $\underline{\textit{72.09}}_{\pm 0.35}$ & \color{red} $\textbf{93.92}_{\pm 0.22}$ & \color{red} $\textbf{86.57}_{\pm 0.27}$ & \color{red} $\textbf{93.50}_{\pm 0.39}$ & \color{red} $\textbf{68.88}_{\pm 0.25}$ & $72.19_{\pm 0.41}$ & \color{red} $\textbf{94.50}_{\pm 0.19}$ & \color{red} $\textbf{86.61}_{\pm 0.29}$ \\
\bottomrule
\end{tabular}
\begin{tablenotes}\footnotesize
\item \footnotesize The best and second-best performing methods on each dataset are shown in {\color{red}\textbf{red}} and {\color{blue}{\underline{\textit{blue}}}}, respectively.
\end{tablenotes}
\end{threeparttable}
}
\end{table*}

\begin{table}
\centering
\caption{Accuracy (in \%) comparison in the citation networks using the graph diffusion convolution and OGBN-proteins.}
\label{tbl:results_citation_networks}
\resizebox{\columnwidth}{!}{
\begin{threeparttable}
\begin{tabular}{r|cccc}
\toprule
\textbf{Model} & \textbf{Cora} & \textbf{Citeseer} & \textbf{Pubmed} & \textbf{OGBN-proteins} \\
\midrule
Cheby \cite{defferrard2016convolutional} & \color{blue} $\underline{\textit{80.79}}_{\pm 1.43}$ & \color{red} $\textbf{69.41}_{\pm 1.09}$ & \color{red} $\textbf{78.25}_{\pm 1.57}$ & \color{blue}$\underline {\textit{63.55}_{\pm 2.89}}$ \\
GCN \cite{kipf2017semi} & $70.22_{\pm 2.91}$ & $58.85_{\pm 2.88}$ & $56.64_{\pm 5.81}$ & $58.86_{\pm 7.63}$ \\
GAT \cite{velickovic2018graph} & $60.29_{\pm 4.25}$ & $56.18_{\pm 2.76}$ & $56.57_{\pm 5.11}$ & $41.58_{\pm 4.42}$ \\
SGC \cite{wu2019simplifying} & $64.23_{\pm 3.66}$ & $56.74_{\pm 2.82}$ & $70.35_{\pm 2.89}$ & $56.87_{\pm 6.42}$ \\
ClusterGCN \cite{chiang2019cluster} & $75.95_{\pm 2.53}$ & $63.93_{\pm 1.96}$ & $59.24_{\pm 6.31}$ & $50.01_{\pm 0.02}$ \\
SuperGAT \cite{kim2021find} & $65.71_{\pm 3.69}$ & $55.80_{\pm 3.27}$ & \textbf{OOM} & $58.59_{\pm 0.79}$ \\
Transformer \cite{shi2021masked} & $66.86_{\pm 3.47}$ & $62.85_{\pm 2.19}$ & $61.83_{\pm 5.46}$ & $39.75_{\pm 4.76}$ \\
GATv2 \cite{brody2022attentive} & $62.16_{\pm 3.29}$ & $53.24_{\pm 2.25}$ & $56.39_{\pm 5.53}$ & $50.06_{\pm 0.13}$ \\
\midrule
S2-GNN (ours) & \color{red} $\textbf{82.47}_{\pm 0.42}$ & \color{blue} $\underline{\textit{68.95}}_{\pm 0.55}$ & \color{blue} $\underline{\textit{77.64}}_{\pm 1.05}$ & \color{red} $\textbf{70.26}_{\pm 4.08}$ \\
\bottomrule
\end{tabular}
\begin{tablenotes}\footnotesize
\item \scriptsize \textbf{OOM}: Out Of Memory in a GPU Nvidia A100 40GB.
\end{tablenotes}
\end{threeparttable}
}
\end{table}

\begin{table}
\centering
\caption{Sparsity (\%) for the non-sparse Sobolev term $\left( \mathbf{L}+\epsilon \mathbf{I} \right)^{\rho}$ for several values of $\rho$.}
\label{tbl:sparsity_percentage}
\resizebox{0.9\columnwidth}{!}{
\begin{threeparttable}
\begin{tabular}{r|ccccc}
\toprule
\textbf{Power} & \textbf{Cancer-B} & \textbf{Cancer-M} & \textbf{20News} & \textbf{HAR} & \textbf{Isolet} \\
\midrule
$\rho=1$ & $96.63$ & $99.17$ & $98.58$ & $98.03$ & $99.01$ \\
$\rho=2$ & $80.33$ & $94.05$ & $65.65$ & $82.29$ & $92.73$ \\
$\rho=3$ & $53.82$ & $82.48$ & $0.615$ & $59.74$ & $75.92$ \\
$\rho=4$ & $29.32$ & $67.67$ & $0$ & $46.05$ & $46.19$ \\
$\rho=5$ & $12.46$ & $52.55$ & $0$ & $35.07$ & $20.41$ \\
\bottomrule
\end{tabular}
\begin{tablenotes}\footnotesize
\item \scriptsize Zero sparsity percentage means that the matrix is completely dense.
\end{tablenotes}
\end{threeparttable}
}
\end{table}


\subsection{Implementation Details}
\label{sec:implementation_details}

S2-GNN and the state-of-the-art methods are implemented using PyTorch Geometric (PyG) \cite{fey2019fast}.
The learned graph methods are implemented using the GSP toolbox \cite{perraudin2014gspbox}.
The value of $k$ in the $k$-NN method is set to $30$. 
Rectified Linear Unit (ReLU) and log-softmax are used as activation functions in all methods.
The hyperparameters such as the number of GNN layers, hidden units, learning rate, weight decay, and dropout rate are optimized for each method using the validation set in each experiment.
For Cheby we optimize the Chebyshev filter size $K$.
For SIGN we optimize the number of powers.
For attention-based techniques like GAT and Transformers, we optimize the number of heads.
For S2-GNN we optimize $\alpha$, $\epsilon$, and whether the fusion mechanism is a linear combination layer or an MLP fusion as introduced in \eqref{eqn:linear_combination_layer} or \eqref{eqn:MLP_fusion_layer}.
The hyperparameter search space for each method is defined as follows: 1) learning rate $lr \in [0.005, 0.02]$, 2) weight decay $wd \in [0.0001, 0.001]$, 3) hidden units of each graph convolutional layer $hu \in \{16, 32, 64\}$, 4) dropout $d \in [0.3, 0.7]$, 5) the number of layers $L \in \{2,3,4,5\}$, 6) Chebyshev filter size $\text{K} \in \{1,2,3\}$, 7) heads of attention mechanism $hd \in \{1,2,\dots,6\}$, 8) powers for SIGN $pw \in \{1,2,3\}$, 9) $\alpha \in \{1,2,\dots,6\}$, 10) $\epsilon \in [0.5, 2]$, 11) fusion layer $fl \in \{\text{Linear}, \text{MLP}\}$.
For the GDC method applied to Cora, Citeseer, and Pubmed datasets, personalized PageRank is used with a teleport probability of $0.05$, and the top $128$ entries with the highest mass per column are selected to sparsify the output of GDC.
The code and constructed graphs will be publicly available upon publication under the MIT license\footnote{\url{https://github.com/jhonygiraldo/S2-GNN}}.

\subsection{Results and Discussion}

Table \ref{tbl:results_constructed_graphs} provides a summary of the comparison between S2-GNN and state-of-the-art methods in terms of Accuracy for Cancer-B, Cancer-M, 20News, HAR, and Isolet.
In most cases, S2-GNN demonstrates the best performance, except for the 20News dataset.
Both Cheby and SIGN methods, which rely on higher-order graph convolutions, also achieve competitive results.
However, the disadvantage of these non-sparse higher-order methods is that the graph matrix representations (either the adjacency or Laplacian matrices) become denser with higher powers of the shift operator.
S2-GNN achieves better performances while maintaining sparsity in the graph matrix representations and providing a good confidence interval in the results.
Furthermore, Table \ref{tbl:results_constructed_graphs} highlights that learned graphs on the right side generally outperform the $k$-NN graphs on the left for most methods (Appendix \ref{app:homophily} presents a study in the homophily level of $k$-NN and learned graphs). 
The performance of the learned graphs is consistently better, with only a few specific cases where the $k$-NN graphs perform better.
However, these instances are rare exceptions.

Table \ref{tbl:results_citation_networks} provides the results of S2-GNN and state-of-the-art methods on the citation network datasets, after being pre-processed with the GDC method, and the OGBN-proteins dataset.
S2-GNN achieves the best performance in Cora and OGBN-proteins datasets and very competitive performance in Citeseer and Pubmed datasets.
Tables \ref{tbl:results_constructed_graphs} and \ref{tbl:results_citation_networks} collectively demonstrate that S2-GNN exhibits strong performance in cases where we construct graphs for the downstream task (using either $k$-NN or smooth graphs) as well as when the graph is naturally provided by the task, as seen in the citation networks.
{We provide additional experiments on graph classification tasks in the Appendix \ref{app:graph_classification}.}


Table \ref{tbl:sparsity_percentage} shows the sparsity percentage of the non-sparse Sobolev term $\left( \mathbf{L}+\epsilon \mathbf{I} \right)^{\rho}$ for different values of $\rho$.
We observe that the non-sparse filter loses sparsity quickly for higher values of $\rho$, resulting in increased memory and computational complexity since more edges are considered for message passing.

{
\noindent \textbf{Memory consumption and inference time}.
We thoroughly analyze the memory consumption and inference time of S2-GNN with synthetic graphs.
To this end, we first build underlying graphs using Erdős–Rényi (ER) models with varying number of nodes $N\in\{500, 1000, 5000, 10000\}$ and edge probabilities $p_{ER}\in\{0.03,0.04,0.05,0.06\}$.
Therefore, by feeding an identical initial feature matrix $\mathbf{X}\in\mathbb{R}^{N\times F}$, where $F=16$, we plotted the averaged memory usage (in terms of GB) and inference times (in seconds) over $30$ random seeds in Figure \ref{fig:time_memory}.
We observe that S2-GNN is among the most efficient models in terms of memory and time complexity when the number of nodes increases and the sparsity of the graphs decreases.
}


\begin{figure*}[t]
    \centering
    \includegraphics[width=0.9\textwidth]{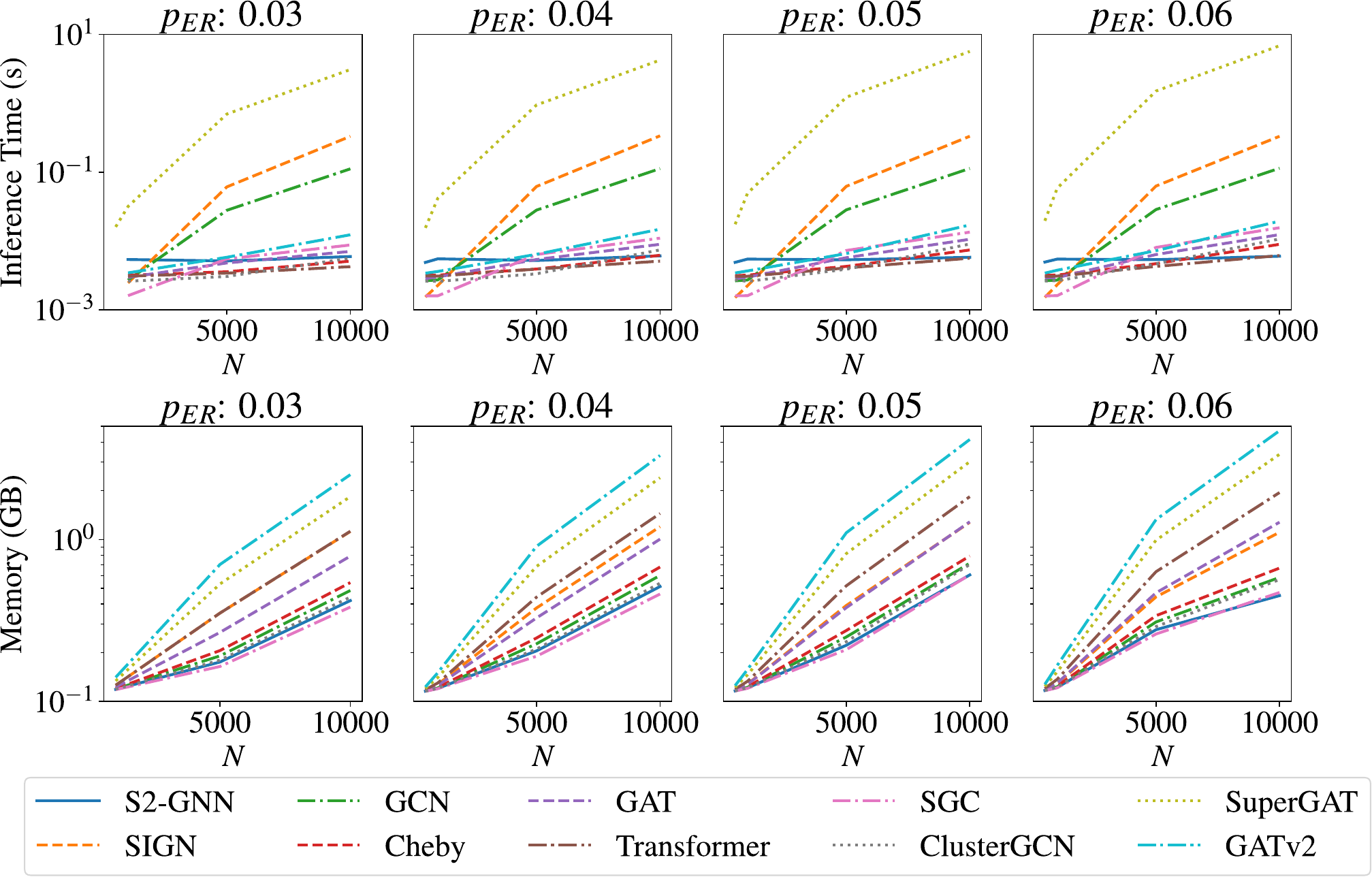}
    \caption{Comparison of inference times (in seconds) and memory consumption (in terms of GB) on the underlying ER graphs with varying number of nodes $N\in\{500,1000,5000,10000\}$ and edge probabilities $p_{ER}\in\{0.03,0.04,0.05,0.06\}$.}
    \label{fig:time_memory}
\end{figure*}

\subsection{Ablation Study}

We conduct an ablation study to investigate whether the performance improvement in S2-GNN is attributed to the additional parameters introduced by having multiple branches or the Hadamard operation in the graph filter.
To explore this, we create a similar model to S2-GNN but exclude the Hadamard powers in each branch.
Instead, we use the identity operation in the first branch and the term $\tilde{\mathbf{D}}^{-\frac{1}{2}}\tilde{\mathbf{A}}\tilde{\mathbf{D}}^{-\frac{1}{2}}$ in the remaining branches.
By performing hyperparameter optimization as outlined in Section \ref{sec:experiments}, we obtain the results summarized in Table \ref{tbl:ablation}.
The findings suggest that, for nearly all datasets, incorporating the Hadamard powers using either $k$-NN or learned graphs leads to better performance compared to simply adding extra parameters to the GNN architecture.
This implies that the Hadamard operation plays a crucial role in enhancing the model's effectiveness.

{
Another ablation experiment studies the usage of the regular and sparse Sobolev norms in S2-GNN. Table \ref{tbl:ablation_regular_vs_sparse_norm} shows the comparison results, where we observe the sparse norm provides better results on all datasets except on 20News.
}

\begin{table*}[]
\centering
\caption{Ablation study in the selection of Hadamard power for the filtering operation.}
\label{tbl:ablation}
\resizebox{\textwidth}{!}{
\begin{threeparttable}
\begin{tabular}{c|ccccc|ccccc}
\toprule
 & \multicolumn{5}{c|}{\textbf{$k$-NN Graphs}} & \multicolumn{5}{c}{\textbf{Learned Graphs}} \\
\multirow{-2}{*}{\textbf{Hadamard Power}} & \textbf{Cancer-B} & \textbf{Cancer-M} & \textbf{20News} & \textbf{HAR} & \textbf{Isolet} & \textbf{Cancer-B} & \textbf{Cancer-M} & \textbf{20News} & \textbf{HAR} & \textbf{Isolet} \\
\midrule
\ding{55} & $91.20_{\pm 0.50}$ & $66.29_{\pm 0.35}$ & $72.08_{\pm 0.37}$ & $93.08_{\pm 0.34}$ & $85.66_{\pm 0.31}$ & $88.49_{\pm 3.67}$ & $67.97_{\pm 0.24}$ & $\textbf{72.63}_{\pm 0.31}$ & $93.82_{\pm 0.25}$ & $86.44_{\pm 0.32}$ \\
\ding{51} & $\textbf{93.39}_{\pm 0.41}$ &  $\textbf{68.48}_{\pm 0.28}$ &  $\textbf{72.09}_{\pm 0.35}$ & $\textbf{93.92}_{\pm 0.22}$ & $\textbf{86.57}_{\pm 0.27}$ & $\textbf{93.50}_{\pm 0.39}$ & $\textbf{68.88}_{\pm 0.25}$ & $72.19_{\pm 0.41}$ & $\textbf{94.50}_{\pm 0.19}$ & $\textbf{86.61}_{\pm 0.29}$ \\
\bottomrule
\end{tabular}
\begin{tablenotes}\footnotesize
\item \footnotesize The best results on each dataset are shown in \textbf{bold}.
\end{tablenotes}
\end{threeparttable}
}
\end{table*}

\begin{table*}[]
\centering
\caption{{Ablation study regarding regular versus sparse Sobolev norm.}}
\label{tbl:ablation_regular_vs_sparse_norm}
\resizebox{\textwidth}{!}{
\begin{tabular}{c|ccccc|ccccc}
\toprule
 & \multicolumn{5}{c|}{\textbf{$k$-NN Graphs}} & \multicolumn{5}{c}{\textbf{Learned Graphs}} \\
\multirow{-2}{*}{\textbf{Sparse Norm}} & \textbf{Cancer-B} & \textbf{Cancer-M} & \textbf{20News} & \textbf{HAR} & \textbf{Isolet} & \textbf{Cancer-B} & \textbf{Cancer-M} & \textbf{20News} & \textbf{HAR} & \textbf{Isolet} \\
\midrule
\ding{55} & $87.05_{\pm 4.25}$ & $64.75_{\pm 0.34}$ & $\textbf{72.29}_{\pm 0.34}$ & $92.80_{\pm 0.30}$ & $85.81_{\pm 0.35}$ & $92.92_{\pm 0.48}$ & $66.41_{\pm 0.37}$ & $\textbf{73.02}_{\pm 0.41}$ & $93.32_{\pm 0.35}$ & $86.25_{\pm 0.30}$ \\
\ding{51} & $\textbf{93.39}_{\pm 0.41}$ &  $\textbf{68.48}_{\pm 0.28}$ &  $72.09_{\pm 0.35}$ & $\textbf{93.92}_{\pm 0.22}$ & $\textbf{86.57}_{\pm 0.27}$ & $\textbf{93.50}_{\pm 0.39}$ & $\textbf{68.88}_{\pm 0.25}$ & $72.19_{\pm 0.41}$ & $\textbf{94.50}_{\pm 0.19}$ & $\textbf{86.61}_{\pm 0.29}$ \\
\bottomrule
\end{tabular}
}
\end{table*}


\begin{figure}
    \centering
    \includegraphics[width=\columnwidth]{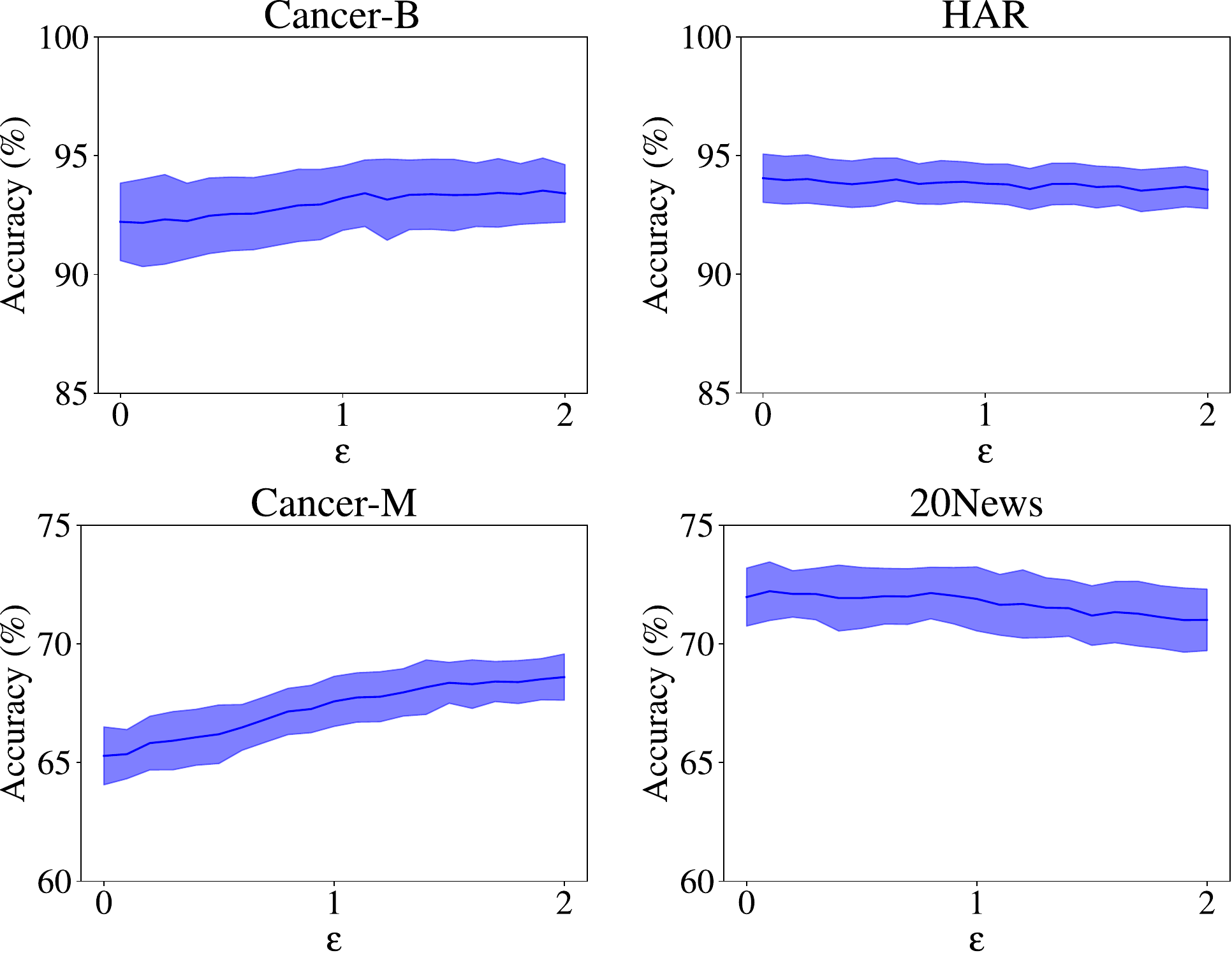}
    \caption{{Mean and standard deviation for the accuracy in different datasets for the sensibility of the hyperparameter $\epsilon$.}}
    \label{fig:sensibility}
\end{figure}


{
\noindent \textbf{Sensibility analysis of $\epsilon$}.
Figure \ref{fig:sensibility} shows a sensibility analysis of the hyperparameter $\epsilon \in [0,2]$ (with standard deviations) for different datasets using $k$-NN graphs.
We observe that $\epsilon$ has a relatively important impact on the Cancer-M dataset.
For the rest of the datasets, S2-GNN is relatively robust to the choice of $\epsilon$.
We note that the sparsity level of the Cancer-M dataset is the highest among the compared datasets.
Therefore, there is a higher chance for the sparsest graph of Cancer-M to have an ill-condition Laplacian and be improved by the Sobolev coefficient $\epsilon$, as discussed in \eqref{eqn:cond_number_sob_and_laplacian}.
}

\subsection{Limitations}

One of the main limitations of S2-GNN is its reliance on weighted graphs as input when using the adjacency matrix.
This requirement can pose a limitation in certain scenarios where only unweighted graphs are available or when the weighted graph construction process is challenging or unreliable.
Additionally, compared to traditional GCN, S2-GNN introduces three additional hyperparameters: $\alpha$, $\epsilon$, and the fusion layer.
The inclusion of these hyperparameters increases the complexity of the model and makes the hyperparameter optimization process more challenging compared to GCN.
Finding an optimal set of hyperparameters for S2-GNN through random search becomes more difficult due to the expanded search space.
These limitations emphasize the importance of careful consideration and understanding of the graph structure and hyperparameter selection when applying S2-GNN in practical scenarios.

\section{Conclusion}
\label{sec:conclusions}

In this work, we introduced a new higher-order sparse graph filter for GNNs.
We proposed a novel sparse Sobolev norm inspired by the GSP literature and used Hadamard powers to compute efficient graph filters.
Sections \ref{sec:sobolev_norm}, \ref{sec:sparse_sob_norm}, and \ref{sec:stability_analysis} provide theoretical insights into our filtering operation. 
To capture higher-order dependencies in GNNs, we introduced the S2-GNN model.
This model incorporates a cascade of increasing sparse filtering operations on each layer, along with a fusion layer that selects the most effective operations.
By leveraging these mechanisms, S2-GNN effectively captures and utilizes higher-order information in graph-structured data.
In addition, we explored the inference of smooth graphs as an alternative when the underlying graph topology is unavailable for a given problem.
Our observations indicate that smooth inferred graphs improve the performance of baseline GNN architectures by promoting homophily.
We thoroughly evaluated S2-GNN in several datasets in a semi-supervised regimen.
S2-GNN achieves competitive performance compared to state-of-the-art GNNs across all tasks while maintaining considerably lower complexity than the best-performing methods.



This work opens several research directions.
The first direction is the extension of S2-GNN to the setting of ClusterGCN \cite{chiang2019cluster}.
This extension would explore the possibility of removing the explicit need for inter-cluster communication, potentially simplifying the model architecture.
Another promising avenue is the study of shift operators that do not require explicit weighted graphs.
This direction would address the limitation of S2-GNN, which requires weighted graphs as input.
By developing alternative shift operators, we can broaden the applicability of the model and overcome this constraint.
Furthermore, neural architecture search can be studied in the context of S2-GNN to approach other problems such as graph classification and link prediction, while addressing the challenge of hyperparameter search in S2-GNN.

\bibliographystyle{IEEEtran}
\bibliography{bibfile}

\clearpage
\newpage

\appendices
\section{Proof of the Theorem \ref{trm:SSob_norm}}
\label{app:proof_SSob_norm}

\begin{proof}
    1) For the triangle inequality we have that:
    \begin{gather}
        \nonumber
        \Vert \mathbf{x+y} \Vert_{(\rho),\epsilon}^2 = \langle \mathbf{x+y}, \mathbf{y+x} \rangle_{(\rho),\epsilon} = \Vert (\mathbf{L}+\epsilon \mathbf{I})^{(\rho/2)} (\mathbf{x+y}) \Vert^2\\
        \label{eqn:intermediate_triangle_inequality}
        \rightarrow \Vert \mathbf{x+y} \Vert_{(\rho),\epsilon}^2 = \Vert \mathbf{x} \Vert_{(\rho),\epsilon}^2 + \Vert \mathbf{y} \Vert_{(\rho),\epsilon}^2 + 2 \langle \mathbf{x}, \mathbf{y} \rangle_{(\rho),\epsilon}.
    \end{gather}
    Using the Cauchy–Schwarz inequality we obtain:
    \begin{gather}
        \nonumber
        \Vert \mathbf{x+y} \Vert_{(\rho),\epsilon}^2 \leq \Vert \mathbf{x} \Vert_{(\rho),\epsilon}^2 + 2\Vert \mathbf{x} \Vert_{(\rho),\epsilon}\Vert \mathbf{y} \Vert_{(\rho),\epsilon} + \Vert \mathbf{y} \Vert_{(\rho),\epsilon}^2\\
        \label{eqn:triangle_inequality}
        \rightarrow \Vert \mathbf{x+y} \Vert_{(\rho),\epsilon}^2 \leq (\Vert \mathbf{x} \Vert_{(\rho),\epsilon} + \Vert \mathbf{y} \Vert_{(\rho),\epsilon})^2.
    \end{gather}
    
    2) Let $s$ be a scalar, then for the absolute homogeneity property we have:
    \begin{equation}
        \resizebox{1\columnwidth}{!}{$
        \Vert s\mathbf{x} \Vert_{(\rho),\epsilon} = \Vert (\mathbf{L}+\epsilon \mathbf{I})^{(\rho/2)} s\mathbf{x}  \Vert = \vert s \vert \Vert (\mathbf{L}+\epsilon \mathbf{I})^{(\rho/2)} \mathbf{x}  \Vert = \vert s \vert \Vert \mathbf{x} \Vert_{(\rho),\epsilon}.$}
        \label{eqn:absolute_homogeneity}
    \end{equation}
    
    3) For the positive definiteness property, we rely on the Schur product theorem, which states that the Hadamard product between two positive definite matrices is also positive definite \cite{horn2012matrix}.
    When $\epsilon=0$ and $\rho=1$, the sparse Sobolev norm reduces to the graph Laplacian quadratic form \cite{shuman2013emerging}.
    In this case, $\Vert \mathbf{x} \Vert_{(1),0}^2 = \mathbf{x}^\mathsf{T}\mathbf{Lx} = 0 \iff \mathbf{x}= \tau \mathbf{1}$ (with $\tau$ some constant value).
    As a result, $\Vert \mathbf{x} \Vert_{(\rho),\epsilon}$ becomes a semi-norm because it satisfies only the first two conditions of vector norms, \ie the triangle inequality and absolute homogeneity.
    However, for $\epsilon > 0$, we have $\mathbf{L}+ \epsilon \mathbf{I} \succ 0$ according to (\ref{eqn:cond_number_sob_and_laplacian}), and then $(\mathbf{L}+ \epsilon \mathbf{I})^{(\rho)} \succ 0$ by using the Schur product theorem. As a result:
    \begin{equation}
        \Vert \mathbf{x} \Vert_{(\rho),\epsilon}^2 = \mathbf{x}^{\mathsf{T}}(\mathbf{L}+\epsilon\mathbf{I})^{(\rho)}\mathbf{x} = 0 \iff \mathbf{x} = 0; ~\forall~\epsilon > 0.
        \label{eqn:positive_definiteness}
    \end{equation}
    As a consequence, the sparse Sobolev norm satisfies the properties of vector norms for all $\epsilon > 0$ according to (\ref{eqn:triangle_inequality}), (\ref{eqn:absolute_homogeneity}), and (\ref{eqn:positive_definiteness}).
    Similarly, the sparse Sobolev norm definition satisfies the properties of semi-norms for $\epsilon=0$ according to (\ref{eqn:triangle_inequality}) and (\ref{eqn:absolute_homogeneity}).
\end{proof}

\section{Proof of Theorem \ref{trm:spectrum_hadamard_product}}
\label{app:proof_spectrum_hadamard_product}

\begin{proof}
    We can compute the spectral decomposition of the Kronecker product between matrices using concepts from product graphs.
    The spectrum of the Kronecker product of multiple graphs is obtained by taking the Kronecker products of their respective spectral components.
    For instance, the spectral decomposition of $\mathbf{L}\otimes \mathbf{L}$ can be expressed as:
    \begin{equation}
        \mathbf{L}\otimes \mathbf{L}=(\mathbf{U} \otimes \mathbf{U})(\mathbf{\Lambda} \otimes \mathbf{\Lambda})(\mathbf{U}^{\mathsf{T}} \otimes \mathbf{U}^{\mathsf{T}}).
        \label{eqn:spectrum_kronecker_product}
    \end{equation}
    Similarly, there is a relationship between the Kronecker and the Hadamard products as follows:
    \begin{theorem}(Visick \cite{visick2000quantitative})
        \label{trm:hadamard_kronecker}
        For $\mathbf{S},\mathbf{T} \in \mathbb{R}^{n\times m}$, we have:
        \begin{equation}
            \mathbf{S} \circ \mathbf{T} = \mathbf{P}_n^{\mathsf{T}}(\mathbf{S} \otimes \mathbf{T})\mathbf{P}_m,
        \end{equation}
        where $\mathbf{P}_n \in \{0,1\}^{n^2 \times n}$ and $\mathbf{P}_m \in \{0,1\}^{m^2 \times m}$ are partial permutation matrices.
        If $\mathbf{S},\mathbf{T} \in \mathbb{R}^{n\times n}$ are square matrices, we have that $\mathbf{S} \circ \mathbf{T} = \mathbf{P}_n^{\mathsf{T}}(\mathbf{S} \otimes \mathbf{T})\mathbf{P}_n$.\\
        Proof: see \cite{visick2000quantitative}.
    \end{theorem}
    Therefore, by using (\ref{eqn:spectrum_kronecker_product}) and Theorem \ref{trm:hadamard_kronecker}, we can derive a general form of the spectrum of the Hadamard product (specifically, for the Hadamard power of order $2$) as follows:
    \begin{equation}
        \mathbf{L} \circ \mathbf{L} = \mathbf{L}^{(2)} = \mathbf{P}_N^{\mathsf{T}} (\mathbf{U} \otimes \mathbf{U})(\mathbf{\Lambda} \otimes \mathbf{\Lambda})(\mathbf{U}^{\mathsf{T}} \otimes \mathbf{U}^{\mathsf{T}})\mathbf{P}_N.
    \end{equation}
    We can then conclude that the spectrum of the Hadamard power of order $\rho$ is such that:
    \begin{equation}
        \mathbf{L}^{(\rho)} = \mathbf{P}_{(\rho)}^{\mathsf{T}} (\mathbf{U} \otimes \mathbf{U}_{(\rho-1)})(\mathbf{\Lambda} \otimes \mathbf{\Lambda}_{(\rho-1)})(\mathbf{U}^{\mathsf{T}} \otimes \mathbf{U}_{(\rho-1)}^{\mathsf{T}})\mathbf{P}_{(\rho)},
    \end{equation}
    where $\mathbf{U}_{(\rho-1)}$ and $\mathbf{\Lambda}_{(\rho-1)}$ are respectively the matrices of eigenvectors and eigenvalues of $\mathbf{L}^{(\rho-1)}$, and $\mathbf{P}_{(\rho)} \in \{0,1\}^{N^2 \times N}$ is a partial permutation matrix for $\mathbf{L}^{(\rho)}$.
\end{proof}





\section{Proof of Theorem \ref{dYY}}
\label{AppC}

\begin{proof}
Firstly, we express the perturbed Laplacian matrix $\hat{\mathbf{L}}$ based on the perturbation model in (\ref{pert_model}) as:
\begin{equation}
\hat{\mathbf{L}}=\hat{\mathbf{D}}-\hat{\mathbf{A}}=(\mathbf{D}+\mathbf{D}_E)-(\mathbf{A}+\mathbf{E}),
\end{equation}
where
\begin{equation}
\mathbf{D}_E=\diag(\mathbf{E1})=\diag(d_{e_1},\hdots,d_{e_N}).
\end{equation}
With the notation of
\begin{equation}
\label{Lk}
\hat{\mathbf{L}}^{(\rho)}=\underbrace{(\hat{\mathbf{L}}+\epsilon\mathbf{I})\circ\hdots\circ(\hat{\mathbf{L}}+\epsilon\mathbf{I})}_{\rho \text{ times}}\text{,}
\end{equation}
the next lemma expresses the elements of $\mathbf{L}^{(\rho)}$ and $\hat{\mathbf{L}}^{(\rho)}$, \ie $\{(\mathbf{L}^{(\rho)})_{ij}\}_{i,j=1}^N$ and $\{(\hat{\mathbf{L}}^{(\rho)})_{ij}\}_{i,j=1}^N$, respectively, based on the edge weights and node degrees as follows:
\begin{lemma}
The elements of $\mathbf{L}^{(\rho)}$ and $\hat{\mathbf{L}}^{(\rho)}$, \ie $\{(\mathbf{L}^{(\rho)})_{ij}\}_{i,j=1}^N$ and $\{(\hat{\mathbf{L}}^{(\rho)})_{ij}\}_{i,j=1}^N$, can be written as:
\begin{equation}
\begin{split}
&\forall i\ne j:\:(\mathbf{L}^{(\rho)})_{ij}=(-1)^{\rho}a^{\rho}_{ij};\:(\hat{\mathbf{L}}^{(\rho)})_{ij}=(-1)^{\rho}(a_{ij}+e_{ij})^{\rho}\\
&(\mathbf{L}^{(\rho)})_{ii}=(d_i+\epsilon)^{\rho};~(\hat{\mathbf{L}}^{(\rho)})_{ii}=(d_i+d_{e_i}+\epsilon)^{\rho}.\\
\end{split}
\end{equation}
\end{lemma}

\begin{proof}
    Using the convention (\ref{Lk}) and zero-diagonality of $\mathbf{A}$, one can write: 
\begin{equation}
\label{Lhat_k}
\begin{split}
\mathbf{L}^{(\rho)}&=\underbrace{(\mathbf{L}+\epsilon\mathbf{I})\circ\hdots\circ(\mathbf{L}+\epsilon\mathbf{I})}_{{\rho} \text{ times}}\\
&=\underbrace{(\mathbf{D}-\mathbf{A}+\epsilon\mathbf{I})\circ\hdots\circ(\mathbf{D}-\mathbf{A}+\epsilon\mathbf{I})}_{{\rho} \text{ times}}\\
&=(\mathbf{D}+\epsilon\mathbf{I})^{\rho}
+(-1)^{\rho}\underbrace{\mathbf{A}\circ\hdots\circ\mathbf{A}}_{{\rho} \text{ times}}\\
&\rightarrow(\mathbf{L}^{(\rho)})_{ii}=(d_i+\epsilon)^{\rho};~(\mathbf{L}^{(\rho)})_{ij}=(-1)^{\rho}a^{\rho}_{ij}.
\end{split}
\end{equation}
Similarly for $\hat{\mathbf{L}}^{(\rho)}$ and by considering the zero-diagonality of $\mathbf{A}$ and $\mathbf{E}$, one can write:
\begin{equation}
\label{Lhat_k22}
\begin{split}
\small \hat{\mathbf{L}}^{(\rho)}&=\underbrace{(\hat{\mathbf{L}}+\epsilon\mathbf{I})\circ\hdots\circ(\hat{\mathbf{L}}+\epsilon\mathbf{I})}_{{\rho} \text{ times}}\\
&\small =\underbrace{(\mathbf{D}+\mathbf{D}_E+\epsilon\mathbf{I}-(\mathbf{A}+\mathbf{E}))\circ\hdots\circ(\mathbf{D}+\mathbf{D}_E+\epsilon\mathbf{I}-(\mathbf{A}+\mathbf{E}))}_{{\rho} \text{ times}}\\
&\small =(\mathbf{D}+\mathbf{D}_E+\epsilon\mathbf{I})^{\rho}+(-1)^{\rho}\underbrace{(\mathbf{A}+\mathbf{E})\circ\hdots\circ(\mathbf{A}+\mathbf{E})}_{{\rho} \text{ times}}\\
& \resizebox{0.92\columnwidth}{!}{$ \rightarrow (\hat{\mathbf{L}}^{(\rho)})_{ii}=(d_i+d_{e_i}+\epsilon)^{\rho};~ (\hat{\mathbf{L}}^{(\rho)})_{ij}=(-1)^{\rho}(a_{ij}+e_{ij})^{\rho}.$}
\end{split}
\end{equation}
\end{proof}

Besides, the next lemma gives upper bounds for $\|\hat{\mathbf{L}}\|$ and Laplacian perturbation distance $\|\mathbf{L}^{(\rho)}-\hat{\mathbf{L}}^{(\rho)}\|$. Note that $r_1(M)$ and $c_1(M)$ denote the maximum Euclidean row and column norms of the matrix $M$.
\begin{lemma}
\label{Lemmaa2} The Laplacian perturbation distance $\|\mathbf{L}^{(\rho)}-\hat{\mathbf{L}}^{(\rho)}\|$ and $\|\hat{\mathbf{L}}\|$ are upper bounded by:
\begin{equation}
\label{norm_Lhat_L2}
\begin{split}
&\|\mathbf{L}^{(\rho)}-\hat{\mathbf{L}}^{(\rho)}\| \le\\
&\sum_{m=1}^{\rho}{\binom{\rho}{m}\left(\eta_m\left\|\mathbf{E}^{(m)}\right\|+(d^{max}+\epsilon)^{\rho-m}(d^{max}_{e})^m\right)},
\end{split}
\end{equation}
and
\begin{equation}
\label{norm_Lhat}
\|\hat{\mathbf{L}}^{(\rho)}\|\le(\hat{d}^{max}+\epsilon)^{\rho}+\|\hat{\mathbf{A}}^{(\rho)}\|,
\end{equation}
where $\hat{d}^{max}=d^{max}+d_{e}^{max}$, $d^{max}=\max_{i}{d_{i}}$, and $d_e^{max}=\max_{i}{d_{e_i}}$.
\end{lemma}

\begin{proof}
For the sake of simplicity, we start with $\|\hat{\mathbf{L}}^{(\rho)}\|$. 
Using (\ref{Lhat_k22}) and by exploiting the triangular inequality principle of the norms, one can write:
\begin{equation}
\label{norm_Lhat}
\begin{split}
&\small \|\hat{\mathbf{L}}^{(\rho)}\|\le \|(\mathbf{D}+\mathbf{D}_E+\epsilon\mathbf{I})^{\rho}+(-1)^{\rho}\underbrace{(\mathbf{A}+\mathbf{E})\circ\hdots\circ(\mathbf{A}+\mathbf{E})}_{{\rho} \text{ times}}\|\\
&\small \le \|(\mathbf{D}+\mathbf{D}_E+\epsilon\mathbf{I})^{\rho}\|+\|\underbrace{(\mathbf{A}+\mathbf{E})\circ\hdots\circ(\mathbf{A}+\mathbf{E})}_{{\rho} \text{ times}: \: \hat{\mathbf{A}}^{(\rho)}}\|\\
&\small \le(\overbrace{d^{max}+d_{e}^{max}}^{\hat{d}^{max}}+\epsilon)^{\rho}+\|\hat{\mathbf{A}}^{(\rho)}\|.
\end{split}
\end{equation}
Regarding $\|\mathbf{L}^{(\rho)}-\hat{\mathbf{L}}^{(\rho)}\|$, it is easier to work on the Laplacian elements. We have:
\begin{equation}
\label{Lhat_ij}
\begin{split}
\forall i\ne j:\:(\mathbf{L}^{(\rho)}-\hat{\mathbf{L}}^{(\rho)})_{ij}&=(-1)^{{\rho}+1}((a_{ij}+e_{ij})^{\rho}-a^{\rho}_{ij})\\
&= (-1)^{{\rho}+1}\left(\sum_{m=1}^{\rho}{\binom{\rho}{m}a^{\rho-m}_{ij}e^m_{ij}}\right),
\end{split}
\end{equation}
\begin{equation}
\label{Lhat_ii}
\begin{split}
(\mathbf{L}^{(\rho)}-\hat{\mathbf{L}}^{(\rho)})_{ii}&=(d_i+\epsilon)^{\rho}-(d_i+\epsilon+d_{e_i})^{\rho}\\
&=-\sum_{m=1}^{\rho}{\binom{\rho}{m}(d_i+\epsilon)^{\rho-m}d^m_{e_i}}.
\end{split}
\end{equation}

Therefore, using (\ref{Lhat_ij}) and (\ref{Lhat_ii}), we can obtain a matrix-form to express $\mathbf{L}^{(\rho)}-\hat{\mathbf{L}}^{(\rho)}$ as:
\begin{equation}
\begin{split}
\small &\mathbf{L}^{(\rho)}-\hat{\mathbf{L}}^{(\rho)}=\\
&\sum_{m=1}^{\rho}{\binom{\rho}{m}\left((-1)^{{\rho}+1}\mathbf{A}^{({\rho}-m)}\circ\mathbf{E}^{(m)}+(\mathbf{D}+\epsilon\mathbf{I})^{({\rho}-m)}\circ\mathbf{D}^{(m)}_E\right)},
\end{split}
\end{equation}
Next, by applying the triangular principle of the norm, one can write:
\begin{equation}
\label{norm_Lhat_L}
\begin{split}
&\|\mathbf{L}^{(\rho)}-\hat{\mathbf{L}}^{(\rho)}\|\le\\
& \sum_{m=1}^{\rho}{\binom{\rho}{m}\left(\left\|\mathbf{A}^{({\rho}-m)}\circ\mathbf{E}^{(m)}\right\|+\left\|(\mathbf{D}+\epsilon\mathbf{I})^{({\rho}-m)}\circ\mathbf{D}^{(m)}_E\right\|\right)}.
\end{split}
\end{equation}

The next Lemma makes it possible to take further steps for simplifying the current bound in (\ref{norm_Lhat_L}).

\begin{lemma}
\label{Lemma_A12}
(Lemma 2c, \cite{zhan1997inequalities}). For any symmetric matrices $\mathbf{A}_1$ and $\mathbf{A}_2$, there is an upper bound on $\|\mathbf{A}_1\circ\mathbf{A}_2\|$ as:
\begin{equation}
\|\mathbf{A}_1\circ\mathbf{A}_2\|\le \min\{c_1(\mathbf{A}_1),r_1(\mathbf{A}_1)\} \|\mathbf{A}_2\|
\end{equation}
\end{lemma}

Using Lemma \ref{Lemma_A12}, we can better simplify the bound in (\ref{norm_Lhat_L}). We first apply Lemma \ref{Lemma_A12} on the first term of the bound in (\ref{norm_Lhat_L}):
\begin{equation}
\label{eq41}
\|(\mathbf{A}^{({\rho}-m)}\circ\mathbf{E}^{(m)})\|\le \overbrace{\min\{r_1(\mathbf{A}^{({\rho}-m)}),c_1(\mathbf{A}^{({\rho}-m)})\}}^{\eta_m}\|\mathbf{E}^{(m)}\|.
\end{equation}
Similarly, the second term can be written as:
\begin{equation}
\label{eq42}
\begin{split}
&\|((\mathbf{D}+\epsilon\mathbf{I})^{({\rho}-m)}\circ\mathbf{D}^{(m)}_E)\|\\
&\le\overbrace{\min\{r_1((\mathbf{D}+\epsilon\mathbf{I})^{({\rho}-m)}),c_1((\mathbf{D}+\epsilon\mathbf{I})^{({\rho}-m)})\}}^{(d^{max}+\epsilon)^{\rho-m}}\overbrace{\|\mathbf{D}^{(m)}_E\|}^{(d^{max}_{e})^m}.
\end{split}
\end{equation}
Finally, by plugging (\ref{eq41}) and (\ref{eq42}) into (\ref{norm_Lhat_L}), the desired result in (\ref{norm_Lhat_L2}) is obtained, and it completes the proof of Lemma \ref{Lemmaa2}. 
\end{proof}

Now, using the triangular principle of the norms, the property of being Lipschitz for the nonlinearity function $\sigma(.)$, and the results in Lemma \ref{Lemmaa2}, we come back to find an upper bound for $d(\mathbf{Y},\hat{\mathbf{Y}})$ in (\ref{d_yhat}) as follows: 
\begin{equation}
\begin{split}
&\|\sigma(\hat{\mathbf{L}}^{(\rho)}\mathbf{X}\hat{\mathbf{W}})-\sigma(\mathbf{L}^{(\rho)}\mathbf{X}\mathbf{W})\|\le\|\hat{\mathbf{L}}^{(\rho)}\mathbf{X}\hat{\mathbf{W}}-\mathbf{L}^{(\rho)}\mathbf{X}\mathbf{W}\|\\
&=\|\hat{\mathbf{L}}^{(\rho)}\mathbf{X}\hat{\mathbf{W}}-\hat{\mathbf{L}}^{(\rho)}\mathbf{X}\mathbf{W}+\hat{\mathbf{L}}^{(\rho)}\mathbf{X}\mathbf{W}-\mathbf{L}^{(\rho)}\mathbf{X}\mathbf{W}\|\\
&\le\|\hat{\mathbf{L}}^{(\rho)}\mathbf{X}\hat{\mathbf{W}}-\hat{\mathbf{L}}^{(\rho)}\mathbf{X}\mathbf{W}\|+\|\hat{\mathbf{L}}^{(\rho)}\mathbf{X}\mathbf{W}-\mathbf{L}^{(\rho)}\mathbf{X}\mathbf{W}\|\\
&\le\|\hat{\mathbf{L}}^{(\rho)}\|~\|\mathbf{X}\|~\|\mathbf{W}-\hat{\mathbf{W}}\|+\|\hat{\mathbf{L}}^{(\rho)}-\mathbf{L}^{(\rho)}\|~\|\mathbf{X}\|~\|\mathbf{W}\|\\
&\le\sqrt{\upsilon}\left(\left[(\hat{d}^{max}+\epsilon)^{\rho}+\|\hat{\mathbf{A}}^{(\rho)}\|\right]\delta_{\mathbf{W}}+\right.\\
&\sum_{m=1}^{\rho}{\binom{\rho}{m}\left(\eta_m\left\|\mathbf{E}^{(m)}\right\|+(d^{max}+\epsilon)^{\rho-m}(d^{max}_{e})^m\right)\|\mathbf{W}\|}).
\end{split}
\end{equation}
Then, using $\|\mathbf{E}\|\le\xi$ and $\|\mathbf{D}_E\|\le\xi_D$, where $\xi$ and $\xi_D$ are sufficiently small, one can use the first-order Taylor approximation (by considering only $m=1$) to further simplify the upper bound as:
\begin{equation}
\begin{split}
&\|\sigma(\hat{\mathbf{L}}^{(\rho)}\mathbf{X}\hat{\mathbf{W}})-\sigma(\mathbf{L}^{(\rho)}\mathbf{X}\mathbf{W})\|\\
&\le\sqrt{\upsilon}\left(\left[(\hat{d}^{max}+\epsilon)^{\rho}+\|\hat{\mathbf{A}}^{(\rho)}\|\right]\delta_{\mathbf{W}}\right.\\
&\left.+{\rho} \left[\eta \|\mathbf{E}\|+(d^{max}+\epsilon)^{\rho-1}d^{max}_{e}\right]\|\mathbf{W}\|\right) + \mathcal{O}(\xi^2)+\mathcal{O}(\xi^2_D).
\end{split}
\end{equation}
and this completes the proof of Theorem \ref{dYY}. Note that, in the last line, we used $\|\mathbf{X}\|\le\|\mathbf{X}\|_F=\sqrt{\upsilon}$.
\end{proof}

{
\section{Experimental Stability Analysis}
\label{app:experiments_stability_analysis}

We validate Theorem \ref{dYY} through experiments by analyzing how the theoretical bounds on the Right Hand Side (RHS) and Left Hand Side (LHS) depend on: 1) the Signal-to-Noise-Ratio (SNR) in both adjacency and weight matrices, 2) the Hadamard order $\rho$, 3) the graph sparsity (relevant to $p_{ER}$), and 4) the Sobolev coefficient $\epsilon$.
More precisely, we build the underlying graph using the ER model with $N$ nodes and edge probability $p_{ER}$.
We then generate true and perturbed elements using \eqref{eqn:perturbations} with varying SNRs $\in\{5,10,20,30,40\}$ to perturb the underlying adjacency matrix and weight matrix $\mathbf{W}$.
Both the initial feature matrix $\mathbf{X}\in\mathbb{R}^{N\times F_0}$ and weight matrix $\mathbf{W}\in\mathbb{R}^{F_0\times F_1}$ are drawn from normal distributions.
In all scenarios, we set $F_0=16$ and $F_1=2$, and the results are averaged over $100$ random seeds.
We analyze the following scenarios:
\begin{itemize}
\item Varying $\rho$: we test $\rho$ in $\{2,3\}$ by letting $N=10$, $p_{ER}=0.3$, and $\epsilon=0.5$.
\item Varying $p_{ER}$: we test $p_{ER}$ in $\{0.1,0.3,0.5\}$ by letting $N=10$, $\rho=2$, and $\epsilon=0.5$.
\item Varying $\epsilon$: we test $\epsilon$ in $\{0.5,5,10\}$ by letting $N=10$, $\rho=2$, and $p_{ER}=0.3$.
\end{itemize}

Figure \ref{fig:experiments_stability} shows the averaged values of RHS and LHS across different SNRs and varying $\rho$, $p_{ER}$, and $\epsilon$.
We find that higher values of $\rho$ make the network more sensitive to perturbations, especially with low SNRs.
Additionally, the network benefits from graph sparsity, as it is more robust to perturbations with lower values of $p_{ER}$.
We also observe that high values of $\epsilon$ can negatively affect the robustness of S2-GNN.
Finally, as the SNR increases, the upper bound becomes tighter, which is related to smaller values of $\delta_{\mathbf{W}}$, $\|\mathbf{E}\|$, $d^{max}_e$, and also $\zeta$ and $\zeta_D$.
}


\begin{figure*}
    \centering
    \includegraphics[width=\linewidth]{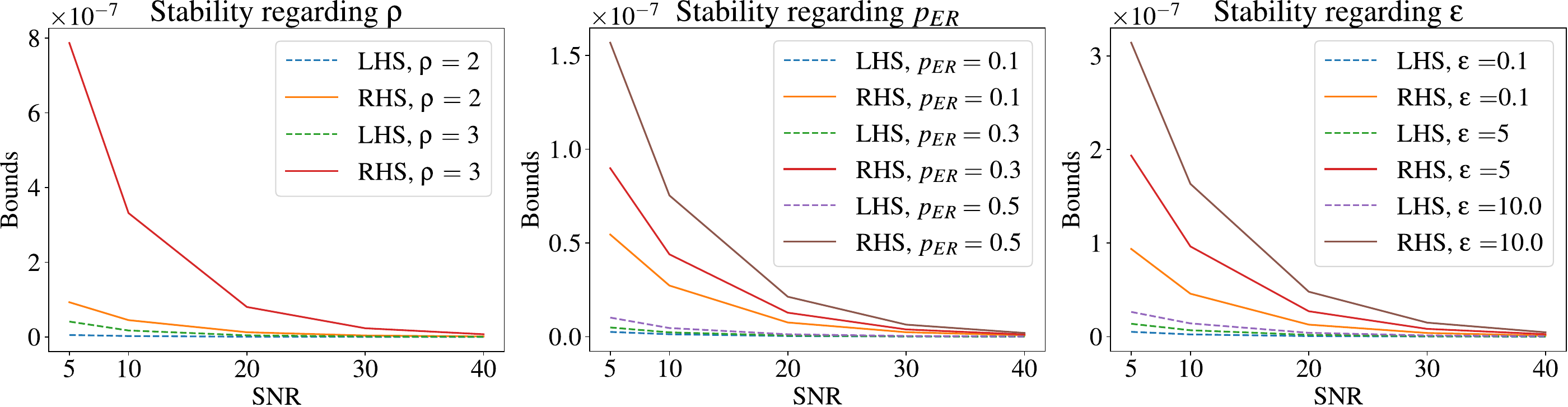}
    \caption{Evaluation of Theorem \ref{dYY} across different configurations for the Sobolev norm when certain perturbations are made to the adjacency matrix.}
    \label{fig:experiments_stability}
\end{figure*}

\section{Homophily Index}
\label{app:homophily}

The homophily index for the datasets used to construct the graphs is presented in Table \ref{tbl:homophily}.
$H(G) \to 1$ when graphs exhibit strong homophily, while $H(G) \to 0$ for graphs with strong heterophily \cite{pei2020geom}.
Thus, low homophily corresponds to high heterophily, and vice versa.
The learned graphs consistently exhibit stronger homophily, which directly correlates with the results in Table \ref{tbl:results_constructed_graphs}, where methods applied to the learned graphs generally perform better.

\begin{table}[]
\centering
\caption{Homophily index $H(G)$ for the constructed graphs with different values of $k$.}
\label{tbl:homophily}
\resizebox{\columnwidth}{!}{
\begin{tabular}{r|cccc|cccc}
\toprule
\multirow{2}{*}{\textbf{Dataset}} & \multicolumn{4}{c|}{\textbf{$k$-NN Graphs}} & \multicolumn{4}{c}{\textbf{Learned Graphs}} \\
                         & $k=10$ & $k=20$ & $k=30$ & $k=40$ & $k=10$ & $k=20$ & $k=30$ & $k=40$ \\
\midrule
\textbf{Cancer-B} & $0.8625$ & $0.8447$ & $0.8334$ & $0.8230$ & $0.8838$ & $0.8670$ & $0.8568$ & $0.8463$ \\
\textbf{Cancer-M} & $0.6267$ & $0.5954$ & $0.5765$ & $0.5629$ & $0.6534$ & $0.6242$ & $0.6054$ & $0.5921$ \\
\textbf{20News} & $0.5557$ & $0.4968$ & $0.4615$ & $0.4354$ & $0.6103$ & $0.5348$ & $0.4919$ & $0.4628$ \\
\textbf{HAR} & $0.8849$ & $0.8478$ & $0.8211$ & $0.8007$ & $0.9087$ & $0.8851$ & $0.8660$ & $0.8500$ \\
\textbf{Isolet} & $0.7346$ & $0.6851$ & $0.6454$ & $0.6105$ & $0.7619$ & $0.7255$ & $0.6943$ & $0.6664$ \\
\bottomrule
\end{tabular}
}
\end{table}

{
\section{Graph Classification Experiments}
\label{app:graph_classification}

We perform additional comparisons between the state-of-the-art methods and S2-GNN for graph classification tasks.
Table \ref{tbl:results_graph_classification} shows the results in three datasets for graph classification: ENZYMES, MUTAG, and PROTEINS \cite{morris2020tudataset}.
We observe that S2-GNN presents competitive results regarding previous state-of-the-art models.
However, we believe that further theoretical and experimental studies should be considered in future work for graph classification.

\begin{table}
\centering
\caption{Accuracy (in \%) comparison for graph classification.}
\label{tbl:results_graph_classification}
\begin{tabular}{r|ccc}
\toprule
\textbf{Model} & \textbf{ENZYMES} & \textbf{MUTAG} & \textbf{PROTEINS} \\
\midrule
Cheby \cite{defferrard2016convolutional} & \color{red} $\textbf{29.39}_{\pm 1.21}$ & $76.11_{\pm 2.54}$ & $67.24_{\pm 0.59}$ \\
GCN \cite{kipf2017semi} & $22.69_{\pm 1.18}$ & $73.14_{\pm 3.86}$ & \color{red} $\textbf{68.15}_{\pm 0.80}$ \\
GAT \cite{velickovic2018graph} & $22.12_{\pm 1.04}$ & $74.98_{\pm 5.53}$ & $64.89{\pm 1.23}$ \\
SGC \cite{wu2019simplifying} & $19.71_{\pm 0.84}$ & $66.73_{\pm 3.86}$ & $61.29_{\pm 0.73}$ \\
ClusterGCN \cite{chiang2019cluster} & $22.56_{\pm 0.81}$ & 
\color{red} $\textbf{81.14}_{\pm 1.11}$ & \color{blue} $\underline{\textit{67.54}}_{\pm 0.57}$ \\
SuperGAT \cite{kim2021find} & $21.97_{\pm 0.97}$ & $71.70_{\pm 6.24}$ & $65.61_{\pm 1.11}$ \\
Transformer \cite{shi2021masked} & \color{blue} $\underline{\textit{24.74}}_{\pm 1.07}$ & \color{blue} $\underline{\textit{80.78}}_{\pm 2.33}$ & $64.71{\pm 1.18}$ \\
GATv2 \cite{brody2022attentive} & $20.78_{\pm 1.24}$ & $71.43{\pm 6.00}$ & $64.47_{\pm 1.13}$ \\
\midrule
S2-GNN (ours) & $23.66_{\pm 0.90}$ & $73.60{\pm 1.46}$ & $67.20_{\pm 0.78}$ \\
\bottomrule
\end{tabular}
\end{table}

}

\end{document}